\documentclass[twocolumn]{IEEEtran} 
\usepackage{times}  
\usepackage{helvet} 
\usepackage{courier}  
\usepackage[hyphens]{url}  
\usepackage{graphicx} 
\usepackage{graphicx}  
\usepackage{subfigure}
\usepackage[utf8]{inputenc} 
\usepackage[T1]{fontenc}    
\usepackage{url}            
\usepackage{booktabs}       
\usepackage{amsfonts}       
\usepackage{nicefrac}       
\usepackage{microtype}      
\usepackage{color}
\usepackage{times}
\usepackage{amsthm,amsmath}
\usepackage{comment}
\usepackage{bbm}
\usepackage[linesnumbered,lined,commentsnumbered,algo2e,ruled]{algorithm2e}
\usepackage{algorithmic}  
\usepackage{amsbsy,amssymb}
\usepackage{cite}

\usepackage[final]{hyperref} 
\usepackage{nameref}

\makeatletter
\let\orgdescriptionlabel\descriptionlabel
\renewcommand*{\descriptionlabel}[1]{%
  \let\orglabel\label
  \let\label\@gobble
  \phantomsection
  \edef\@currentlabel{#1}%
  \let\label\orglabel
  \orgdescriptionlabel{#1}%
}
\makeatother

\newtheorem{proposition}{Proposition}
\newtheorem{lemma}{Lemma}
\newtheorem{theorem}{Theorem}
\newtheorem{definition}{Definition}

 \allowdisplaybreaks

\title{The Restless Hidden Markov Bandit\\with Linear Rewards and Side Information}
\author{
	Michal Yemini\thanks{
	Department of Electrical Engineering. Stanford University, USA. \texttt{michalye@stanford.edu}},
	Amir Leshem\thanks{
	Faculty of Engineering, Bar-Ilan University,  Ramat Gan, Israel 	\texttt{leshem.amir2@gmail.com}. Amir Leshem was partially supported by ISF grants ISF 2277/16 and ISF 1644/18.}
	and
	Anelia Somekh-Baruch\thanks{
	Faculty of Engineering
	Bar-Ilan University, Ramat Gan, Israel
	\texttt{somekha@biu.ac.il}}
	\thanks{A summary of the results presented in this paper was accepted to the 59th Conference on Decision and Control.}
}

\begin{document}
	
	\maketitle
	
	\begin{abstract}
	In this paper we present a model for the hidden Markovian bandit problem with linear rewards.  As opposed to current work on Markovian bandits, we do not assume that the state is known to the decision maker before making the decision. Furthermore, we assume structural side information where the decision maker knows in advance that there are two types of hidden states; one is common to all arms  and evolves according to a Markovian distribution, and the other is unique to each arm and is distributed according to an i.i.d. process that is unique to each arm. 
	We present an algorithm and regret analysis to this problem.
	Surprisingly, we can recover the hidden states and maintain logarithmic regret in the case of a convex polytope action set. Furthermore, we show that the structural side information leads to expected regret that does not depend on the number of extreme points in the action space. Therefore, we obtain practical solutions even in high dimensional problems.
	\end{abstract}
	
	\section{Introduction}
	\subsection{Preliminaries}\label{sec:prob_preliminaries}
	This work considers a setup in which at each time instant $t$, a decision maker chooses an arm $b_t\in\mathcal{B}\subset\mathbb{N}$ to pull and an action $\boldsymbol a_t\in\mathcal{A}\subset\mathbb{R}^N$ and gets a reward that depends linearly on a random function of an unknown  system state  
	$s_t\in\mathcal{S}\subset\mathbb{N}$  and the chosen arm $b_t\in\mathcal{B}$ in the following way: 
	\begin{flalign}\label{eq:noiseless_reward}
	r_t(b_t,\boldsymbol a_t) = \left\langle\boldsymbol a_t,\boldsymbol \theta(b_t,s_t)\right\rangle,
	\end{flalign}
	where  $\left\langle \boldsymbol x,\boldsymbol y\right\rangle$ is the inner product between  $\boldsymbol x$ and $\boldsymbol y$. The reward function $r_t(b_t,\boldsymbol a_t)$ depends on two types of system states $s_t$ and $\boldsymbol \theta(b_t,s_t)$. The first type of state, denoted by $s_t$, is common to all arms and represents a ``global" system state, whereas the second system state $\boldsymbol \theta(b_t,s_t)$ depends on both the ``global" system state $s_t$ and the arm chosen $b_t$.
	We assume that the set $\mathcal{A}$ is compact and that the sets $\mathcal{S}$ and $\mathcal{B}$ are finite.
	The process $(s_t)_{t=1,2,\ldots}$ is a time-homogeneous, irreducible and aperiodic Markov chain over a finite state space $\mathcal{S}$. 
	Additionally, for each $s\in\mathcal{S}$ and $b\in\mathcal{B}$, $\boldsymbol\theta(b,s)$ is a random function with a range $\boldsymbol\Theta_{b,s}\subset\mathbb{R}^N$.
	We refer to this model as the \textit{restless hidden Markov bandit model with a linear reward}. In this work we consider a setup in which a decision maker only knows the sets $\mathcal{A},\mathcal{S}$ and $\boldsymbol\Theta_{b,s}$ for all $(b,s)\in\mathcal{B}\times\mathcal{S}$ but neither the transition probabilities of the Markov process $(s_t)_{t=1,2,\ldots}$, the probability distribution of the random function $\boldsymbol \theta(b, s)$, nor their realizations at time $t$.    
	This model captures, for example,  an uplink cognitive radio communication network with a wide-band primary user and a narrow-band secondary user, we depict this system in Figure \ref{fig:system_exmaple}. The wide-band primary user  communicates in the  2.4GHz band and if active uses all   83MHz available bandwidth. The presence of the primary user is modeled by a Gilbert–Elliott model \cite{4518280} that comprises a Markov chain $S$ with two binary states, where state  $s=1$ denotes a transmitting primary user and $s=0$ denotes a vacant channel, i.e., inactive primary user. To limit the interference to the primary user, a secondary user may choose to transmit over one of three possible  20MHz channels, namely channels 1,6 or 11, these channels are describe by the arm set $\mathcal{B}$. Additionally, each of these   channels is partitioned into $N=100$ sub-bands of $200$KHz. After choosing a frequency band (depicted by an arm choice $b$) for transmission, the secondary user chooses a frequency hopping sequence (denoted by the vector $\boldsymbol{a}$)  over the sub-bands included in the chosen channel, subject to a total time utilization constraint over the sub-bands. Upon making the choices, the secondary user does not know what the current rates of the chosen frequency bands is i.e., the vector $\boldsymbol\theta$, but only knows the communication rates  of previous transmissions. At the end of a transmission the receiver sends the achieved communication rate to the secondary user via a high-capacity backhaul link.\footnote{We note that the receiver does not send the secondary user the vector $\boldsymbol\theta$ since the secondary user may not use all the frequency bands in a channel, this is especially true in a high dimensional setup, i.e., $N\gg1$ where reducing the number of bands whose state the receiver needs to estimate is paramount to reducing power consumption, delays and decoding complexity scheme.}  
	
	\begin{figure}
	    \centering
	    \hspace{-0.8cm}
	    \includegraphics[scale=0.35,trim=5.8cm 2.5cm 4cm 1cm, clip]{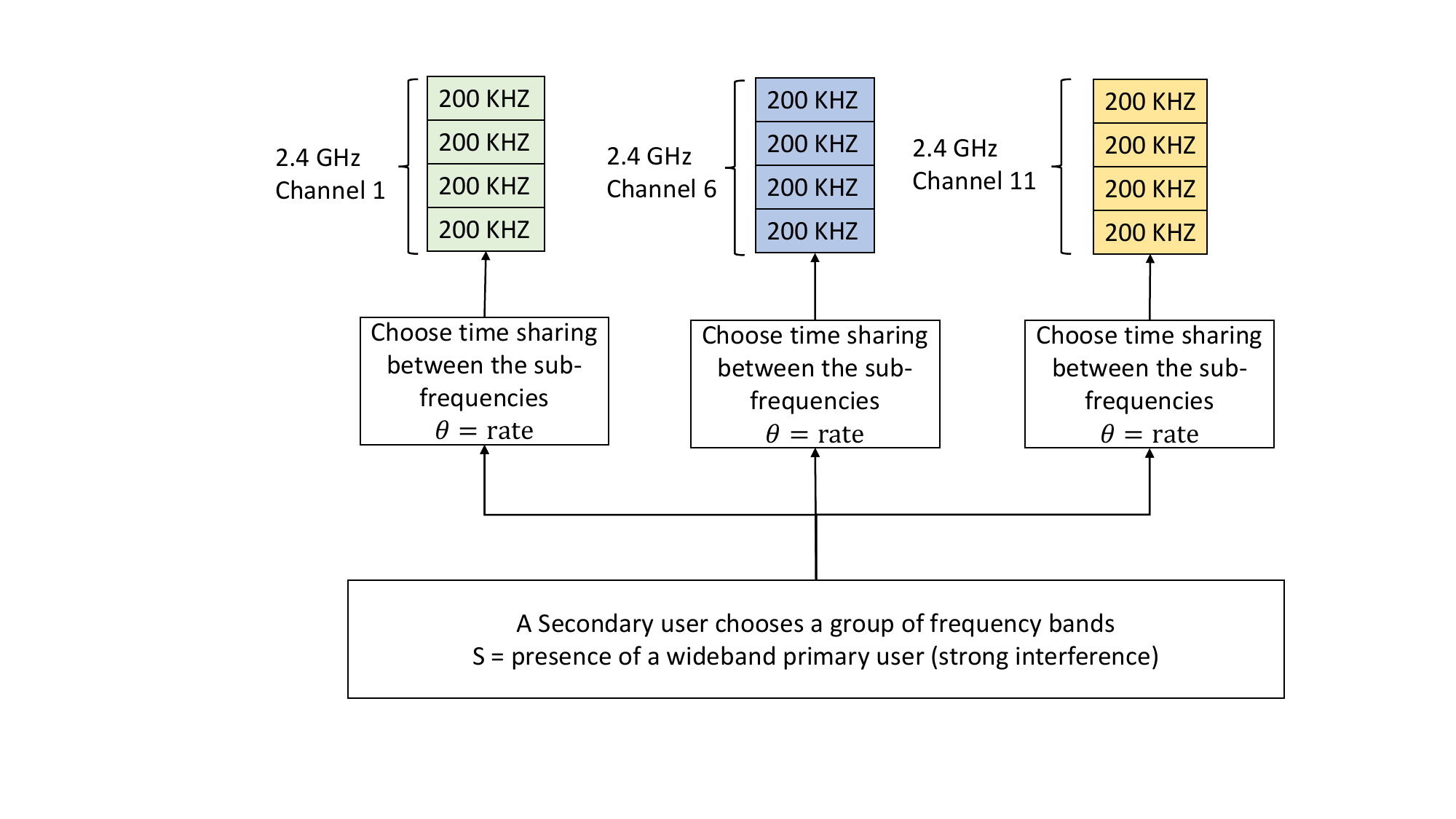}
	    \caption{An example of a cognitive communication network.}
	    \label{fig:system_exmaple}
	\end{figure}
	  
	\subsection{Discussion}
	The restless hidden Markov bandit model with linear rewards described in (\ref{eq:noiseless_reward}) is related to several learning models, among them are   Markov bandits and Markov decision processes, stochastic linear bandits, restless bandits models,  partially observed Markov decision processes, and bandits with structural side information.
	In Markov bandit models the reward is generated by each arm independently of other arms, and changes over time according to a Markov process that progresses over time when an arm is played. In \cite{1104485}, an analysis of the expected regret of a policy for the Markov bandit model was performed. It compared the expected reward of the policy to that of the arm with the best expected reward that was found based on the stationary distribution of the Markov chain.
	In the Markov decision process (MDP) literature, which considers a Markov process that controls the state of a system, it is assumed that this state is known to the decision maker upon choosing an action to play \cite{5510097,6200864,Auer07logarithmiconline, Bartlett2009REGALAR,Jaksch:2010:NRB:1756006.1859902,Qian2018ExplorationBF,Fruit:2018}. In the work \cite{1104485} the decision maker maximizes its expected reward by playing the arm with the maximum expected reward based on the stationary distribution, therefore the optimal policy in \cite{1104485} does not maximize reward based on the current system state. An instantaneous in time minimization regret approach which depends on the current system state and the transition probabilities of the Markov chain, and not only its stationary distribution  has been extensively investigated in many papers.  Among them are \cite{5510097,6200864,Auer07logarithmiconline} which attain a logarithmic regret that depends on the Markov chain parameters and the size action space, assuming that the Markov chain is a unichain.  Additionally, the general case that includes weakly communicating Markov chains is investigated in \cite{Bartlett2009REGALAR,Jaksch:2010:NRB:1756006.1859902,Qian2018ExplorationBF,Fruit:2018} and \cite{DBLP:journals/corr/abs-1808-01813} in which a scheme that achieves an $\tilde{O}(\sqrt{T})$ regret that depends on the MDP parameters is considered. It was also proven in \cite{Fruit:2018} that it is not possible to achieve a logarithmic regret with a polynomial dependency on the MDP parameters, assuming no prior knowledge regarding the bias span is available. These results were derived under the assumption that the decision maker knows the state of the Markov process before choosing an action, additionally, it is assumed that the transition probabilities of the Markov chain depend on the action played, and that no side information regarding the reward function is known.
	
	The restless hidden Markov bandit model with linear rewards
is also related to stochastic linear bandits, which study the model 
in which at each time $t$ a decision maker chooses an action vector $\boldsymbol a_t$ from a predefined set and receives a reward that is a linear function of the action vector, i.e., $r_t=\left\langle\boldsymbol a_t,\boldsymbol \theta_t\right\rangle+w_t$, (see for example \cite{Auer:2003:UCB:944919.944941,Dani2008StochasticLO,doi:10.1287/moor.1100.0446,7472588}).
It is assumed that $\boldsymbol \theta_t$ is unknown and that $w_t$ is a random noise. Using confidence bounds and optimism in face of uncertainty the aforementioned works derived expected regret bounds under several assumptions regarding the probability distributions of the vectors $\boldsymbol \theta_t$ and the noise $w_t$, and the action set of the decision maker.        
Another related model is  the contextual bandit model with expected linear rewards \cite{10.1145/1772690.1772758,pmlr-v15-chu11a,DBLP:journals/corr/Zhou15c} where at each time instant the  reward of arm $b$ in a set of arms $\mathcal{B}$ is a random function $r_{t,b}=\left\langle\boldsymbol a_{t,b},\boldsymbol \theta\right\rangle+w_t$. The vector $\boldsymbol a_{t,b}$ is a context vector that is revealed to the decision maker at each time $t$ before choosing an arm $b$ to play, and $\boldsymbol\theta$ is an unknown vector that the decision maker tries to estimate (exploration) while aiming at maximizing the  total reward (exploitation).

	The restless hidden Markov bandit model with linear rewards is also related to the restless Markov bandits investigated, for example, in \cite{6120273,ORTNER201462,Meshram2017,8695066}. In this setup, the process that governs the arms constantly evolves regardless of which arm is pulled. We note that our model is different from \cite{Meshram2017} since our setup assumes that the Markov chain that governs the system states is common to all arms, however, the arm selection affects the reward received for the current system state. Another relevant model is the partially observed Markov decision process (POMDP) \cite{Kaelbling1998PlanningAA,10.2307/40538383}. In this model, 
 a decision maker aims to maximize its expected accumulative reward and has to balance its desire to increase the immediate reward with the benefits of improving the belief of the unknown state of the system. Other related POMDP models include \cite{969499,1200155,7799449}. The works \cite{969499,1200155} consider a tracking problem with independent objects and uses an approximated Gittins index approach for finding policies. In \cite{7799449}  an  information acquisition and sequential belief refinement with a finite number of possible actions is considered. Finally, our model is also related to
	the Gaussian mixture models for the multi-armed contextual bandit model considered in \cite{Urteaga2018NonparametricGM} in which the reward distributions are approximated using nonparametric Gaussian mixture models. A notable difference is, however, that  since the structure of the reward function is known to be linear in our model, we estimate the probability distribution of the system instead of estimating the reward distribution for each action.
	
	Finally, the decision maker in the restless hidden Markov bandit model with linear rewards has a structural side information, that is, the decision maker knows in advance that the reward function is linear and that the hidden states are composed of two types: a state that is common to all arms, and a state that depends on the arm played. However, the decision maker does not know the mean reward or exact probability distribution of each of these hidden states. We show that the decision maker can take advantage of this side information regarding the reward function, i.e., its linearity and the two types of states, to maximize the expected reward. Therefore, the restless hidden Markov bandit model with linear rewards relates to learning  problems with structural side information see \cite{agrawal1995continuum,kleinberg2008multi,magureanu2014lipschitz,slivkins2011multi,combes2014unimodal,8798887,combes2017minimal,NIPS2018_8103,8619134}. The papers \cite{8798887,agrawal1995continuum,kleinberg2008multi,magureanu2014lipschitz,slivkins2011multi,combes2014unimodal,combes2017minimal} consider multi arm bandit problems and the papers \cite{NIPS2018_8103,8619134} consider MDPs. In particular, in regards to MDPs, the work \cite{NIPS2018_8103}  considers MDPs with structural side information.  It presents explicit regret bounds and tractable algorithms for two special models, the first model is an MDP with no structural side information and the second model is an MDP with transition probability and mean reward function that are Lipschitz functions of the state and action spaces in an embedded Euclidean space. In the case of Lipschitz structured MDPs the expected regret function may be independent of the cardinality of the action and state spaces, however, this is achieved by increasing the power of the Markov chain span coefficient, introducing additional variables that depend on the embedded Euclidean space and calculating the optimal policy at each time instant and not in increasing intervals of time. The paper \cite{8619134} assumes side information that upper bounds the maximal difference of the transition probabilities beginning at two different states. These two states are called similar if the bound is small. The learning of the transition matrix of the MDP is then accelerated by using samples of transitions from one state in the estimation of the transition matrix of all its similar states. Note that our side information model differs from the Lipschitz structured MDP discussed in \cite{NIPS2018_8103} to which tractable algorithm is proposed. Additionally, our algorithm does not require the computation of the optimal policy every time instant, this is especially important when power consumption is considered. 
	 Furthermore, we do not assume that the decision maker knows in advance the upper bounds on the maximal difference of the transition probabilities beginning at two different states that are known in \cite{8619134}. Finally, we note that contrary to these works our setup also assumes that the decision maker has to infer the previous state from the previous actions and rewards. 
	 
	\textit{Contributions:} This work differs from the aforementioned works in several aspects: First, it is neither a classical stochastic linear bandit process since the state of the system evolves over time according to a Markov process.
	Furthermore, it is not a classical Markov bandit model nor a restless one since the states are not directly observed or given to the decision maker. Interestingly, we prove that the uncountability of the  action space or the cardinality of its set of extreme points does not affect the expected regret.  The scheme we propose takes advantage of structural side information regarding the problem and divides the estimation of the probability distribution which controls the system evolution into two parts; the first part estimates the transition matrix of the Markov chain common to all arms, and the second estimates the probability distributions of the unknown  parameter $\boldsymbol\theta$, instead of the expected reward, which depends on the system state and the arm. Our numerical results demonstrate the merits of our proposed scheme, namely the significant reduction of the expected regret of the decision maker as compared to an algorithm such as the UCRL algorithm \cite{Auer07logarithmiconline} that ignores structural side information.

	
	\section{Model Formulation}\label{sec:prob_for}
	This section defines the restless hidden Markov bandit problem with a linear reward in more detail. We consider the setup that is stated in (\ref{eq:noiseless_reward}).
	The process $(s_t)_{t=1,2,\ldots}$ is a finite space $\mathcal{S}$ Markov chain with a transition matrix $P_{S}$  and a unique stationary distribution  $\mu_{S}$. Our analysis holds for any distribution of the initial state $s_{t=1}$ and thus we do not make any assumption regarding the distribution of the initial state $s_{t=1}$.
	We denote the transition probability between state $\tilde s$ and $\check s$ in $\mathcal{S}$ by $P_{S}(\tilde s,\check s)$. Let $\mathcal{B}$ be the set of arms that the decision maker can choose from and
	let the action space $\mathcal{A}\subset\mathbb{R}^N$ of a decision maker be an $N$-dimensional compact and convex polytope. 
The set $\mathcal{A}$ represents the possible  resource  allocations to $N$ random processes that are captured by the $N$-dimensional random vectors  $\boldsymbol\theta$ at time $t=1,2,\ldots$, each process is depicted by a coordinate in the vector $\boldsymbol\theta$. 
	We denote the set of extreme  points (also known as vertices) of $\mathcal{A}$ by $\boldsymbol V$.
	We also assume that the set $\boldsymbol\Theta_{b,s}$ is finite for every $(b,s)\in\mathcal{B}\times\mathcal{S}$ and that $|\boldsymbol V|\gg |\boldsymbol\Theta_{b,s}|$. Additionally,  $\boldsymbol\Theta_{b,\tilde{s}}\cap\boldsymbol\Theta_{b,\check{s}}=\emptyset$ for every $\tilde{s}\neq \check{s}$. Finally, we denote by $P_{\boldsymbol\Theta_{b,s}}(\boldsymbol \theta)$ the probability distribution of the random vector $\boldsymbol \theta\in\boldsymbol\Theta_{b,s}$. 
	
	At each time $t$ the decision maker receives a reward $r_t = \left\langle\boldsymbol a_t,\boldsymbol \theta(b_t,s_t)\right\rangle$. Upon receiving this reward the decision maker chooses an arm $b_{t+1}$ to pull and an action choice $\boldsymbol a_{t+1}$, given the arm choices, actions and rewards of previous times, $1,\ldots,t$, and the sets $\mathcal{B},\mathcal{A},\mathcal{S}$ and $\boldsymbol\Theta_{b,s},\:(b,s)\in\mathcal{B}\times\mathcal{S}$. A key difference between the restless hidden Markov bandit model with linear rewards and the contextual bandit model with linear rewards  \cite{10.1145/1772690.1772758,pmlr-v15-chu11a,DBLP:journals/corr/Zhou15c} is that in the latter the vector $\boldsymbol{a_t}$ is a context vector that depicts the state of the system for each arm and that the decision maker cannot control, the decision maker chooses an arm $b_t$ in effort to estimate the vector $\boldsymbol\theta$ to maximize its total expected reward. 
	In our model the random system state at time $t$ is not provided to the decision maker, this system state is captured by the random pair $(s_t,\boldsymbol\theta_{b_t,s_t})$, where $s_t$ depicts the system evolution in time and $\boldsymbol\theta$ captures the unique characteristics of each arm given a system state. Furthermore, the decision maker in our model has another degree of freedom in the choice of the action vector $\boldsymbol a_t\in\mathcal{A}$.
	
	We define the regret of the \textit{hidden} restless Markov bandit model with linear rewards with respect to the expected reward of the restless Markov bandit model with linear rewards. This model assumes that the decision maker perfectly knows in advance all the parameters of the model as well as the identity of the previous state.
	
	\begin{definition}[The Restless Markov Bandit Model with Linear Rewards]
		In the restless Markov bandit model with  linear rewards a decision maker knows in advance the transition matrix $P_{S}$ and the probability distributions $P_{\boldsymbol\Theta_{b,s}},\: (b,s)\in\mathcal{B}\times\mathcal{S}$ as well as the sets $\mathcal{A},\mathcal{S}$ and $\boldsymbol\Theta_{b,s},\:(b,s)\in\mathcal{B}\times\mathcal{S}$. 
		\newline
		At each time $t$ the decision maker receives a reward $r_t = \left\langle\boldsymbol a_t,\boldsymbol \theta(b_t,s_t)\right\rangle$ and observes the identity of the state $s_t$. Upon receiving this reward the decision maker chooses an arm $b_{t+1}$ to pull and an action choice $\boldsymbol a_{t+1}$ given the actions, rewards and states of previous times, $1,\ldots,t$, and the sets $\mathcal{A},\mathcal{S}$ and $\boldsymbol\Theta_{b,s},\:(b,s)\in\mathcal{B}\times\mathcal{S}$ as well as the transition matrix $P_S$ and the probability distributions $P_{\boldsymbol\Theta_{b,s}},\: (b,s)\in\mathcal{B}\times\mathcal{S}$. 
	\end{definition}

	\begin{definition}[The Average Reward of a Policy for the Restless Markov Bandit Model with Linear Rewards]
		A policy for the restless Markov bandit model with linear rewards is defined as a mapping $\pi:\mathcal{S}\rightarrow\mathcal{B}\times\mathcal{A}$. 	
		The average reward of an action policy $\pi$ is defined as
		\begin{flalign}\label{eq:Markov_average_reward_def}
		\hspace{-0.2cm}\rho(\pi) = \hspace{-0.1cm}\sum_{\tilde s,\check s\in\mathcal{S}}\mu_S(\tilde s)P_{S}(\tilde s,\check s)\hspace{-0.1cm}\sum_{\boldsymbol\theta\in\boldsymbol\Theta_{b_{\pi}(\tilde{s}),\check s}}\hspace{-0.1cm}
		P_{\boldsymbol\Theta_{b_{\pi}(\tilde{s}),\check s}}(\boldsymbol\theta)\left\langle\boldsymbol a_{\pi}(\tilde s),\boldsymbol \theta\right\rangle.
		\end{flalign}
	\end{definition}
	
	\begin{definition}[Regret Definition for the Restless Hidden Markov Bandit Model with Linear Rewards]
		Denote by  $\pi^*$ the policy that maximizes (\ref{eq:Markov_average_reward_def}).
		Recall the reward definition (\ref{eq:noiseless_reward}) for the restless hidden Markov bandit model with linear rewards. We define the regret of the restless hidden Markov bandit model with linear rewards as
		\begin{flalign}\label{eq:regret_df_compared_to_rho}
		R(T) = T\rho(\pi^*)-\sum_{t=1}^T r_t(b_t,\boldsymbol a_t).
		\end{flalign}
		That is, we define the regret to be relative to the optimal policy for the scenario in which the decision maker is in possession of the previous state, the Markov chain transition matrix and the probability distributions of $\boldsymbol \theta$. 
	\end{definition}
\textit{Notation:} We denote by $B(\boldsymbol c,r)$ the $n$-dimensional Euclidean ball with center $\boldsymbol c\in\mathbb{R}^N$ and radius $r$. 	
\section{Upper Confidence Bound Reinforcement Learning for the Restless Hidden Markov Bandit Model with Linear Rewards}
This section presents Algorithm \ref{algo:UCRL_instantaneous} and establishes its expected regret for the  restless hidden Markov bandit model with linear rewards.
Algorithm \ref{algo:UCRL_instantaneous} uses two types of upper confidence bounds, the  first assists in estimating the transition probabilities of the Markov chain and is not arm dependent, the second assists in estimating the probability distributions $P_{\boldsymbol\Theta_{b,s}}$ which are arm dependent. Additionally, Algorithm \ref{algo:UCRL_instantaneous}  recovers at each time $t$ the identity of the previous state $s_{t-1}$ with probability 1, this recovery is $\epsilon$-optimal in the sense that for each $\epsilon>0$ we can find a recovery scheme with an expected regret smaller than $\epsilon$.

\subsection{The Motivation for Algorithm \ref{algo:UCRL_instantaneous}}
Let $P_S$ be a transition matrix of a Markov chain with state set $\mathcal{S}$, let $\mathcal{B}$ be a finite set and let  $P_{\boldsymbol\Theta_{b,s}},\: (b,s)\in\mathcal{B}\times\mathcal{S}$  be probability distributions. 
Additionally, denote
\begin{flalign}\label{eq:A_star_def}
&\mathcal{A}^*(P_S,\{P_{\boldsymbol\Theta_{b,s}}\}_{ (b,s)\in\mathcal{B}\times\mathcal{S}},\tilde{s})\triangleq\Bigg\{ (b_{*},\boldsymbol a_{*})\in\mathcal{B}\times\boldsymbol{V}:\nonumber\\
&(b_{*},\boldsymbol a_{*})\hspace{-0.05cm}\in\hspace{-0.05cm}\arg\hspace{-0.2cm}\max_{b\in\mathcal{B},\boldsymbol a \in\boldsymbol{V}}\Bigg\{\sum_{\check{s}\in\mathcal{S}}
\hspace{-0.1cm}P_S(\tilde{s},\check{s})\hspace{-0.1cm}\sum_{\boldsymbol\theta\in\boldsymbol\Theta_{b,\check s}}
\hspace{-0.2cm}P_{\boldsymbol\Theta_{b,\check s}}(\boldsymbol\theta)\left\langle\boldsymbol a,\boldsymbol \theta\right\rangle\Bigg\}\hspace{-0.1cm}\Bigg\}.
\end{flalign}

\begin{lemma}\label{lemma:equal_vicinity}
	For every  transition matrix $P_{S}$ and every collection of probability distributions $P_{\boldsymbol\Theta_{b,s}},\: (b,s)\in\mathcal{B}\times\mathcal{S}$ there exists $\delta>0$ such that if
	$|\hat{P}_{S}(\tilde{s},\check{s})-P_S(\tilde{s},\check{s})|\leq\delta$ for all $\tilde{s},\check{s}\in\mathcal{S}$, and
	$|\hat{P}_{\boldsymbol\Theta_{b,s}}(\boldsymbol\theta)-P_{\boldsymbol\Theta_{b,s}}(\boldsymbol\theta)|\leq\delta$ for all $\boldsymbol\theta\in\boldsymbol\Theta_{b,s}$ and $(b,s)\in\mathcal{B}\times\mathcal{S}$ 
	we have that
	$\mathcal{A}^*\left(\hat{P}_S,\{\hat{P}_{\boldsymbol\Theta_{b,s}}\}_{ (b,s)\in\mathcal{B}\times\mathcal{S}},\tilde{s}\right)= \mathcal{A}^*\left(P_S,\{P_{\boldsymbol\Theta_{b,s}}\}_{ (b,s)\in\mathcal{B}\times\mathcal{S}},\tilde{s}\right)
	$
	for every $\tilde{s}\in\mathcal{S}$.
\end{lemma}

\begin{proof}
	Recall that $\boldsymbol V$ is the set of extreme points of the polytope $\mathcal{A}$. Since the set $\mathcal{A}$ is a bounded and convex polytope, the optimal actions of $\sum_{\check{s}\in\mathcal{S}}
	P_S(\tilde{s},\check{s})\sum_{\boldsymbol\theta\in\boldsymbol\Theta_{b,\check{s}}}
	P_{\boldsymbol\Theta_{b,\check s}}(\boldsymbol\theta)\left\langle\boldsymbol a,\boldsymbol \theta\right\rangle$ lie in the set $\boldsymbol V$ for every choice of arm $b\in\mathcal{B}$ and state $\tilde{s}\in\mathcal{S}$. 
	
	Denote $\mathcal{P}_{\boldsymbol\Theta}=\{P_{\boldsymbol\Theta_{b,s}}\}_{ (b,s)\in\mathcal{B}\times\mathcal{S}}$ and let \[g(b,\boldsymbol{a},\tilde{s}) \triangleq \sum_{\check{s}\in\mathcal{S}}
	P_S(\tilde{s},\check{s})\sum_{\boldsymbol\theta\in\boldsymbol\Theta_{b,\check{s}}}
	P_{\boldsymbol\Theta_{b,\check s}}(\boldsymbol\theta)\left\langle\boldsymbol a,\boldsymbol \theta\right\rangle.\]
	The set of arms $\mathcal{B}$ is finite and bounded, and the set of states $\mathcal{S}$ and $\boldsymbol\Theta_{b,s}$  is finite as well, therefore, $\max\left\{|\langle\boldsymbol\theta,\boldsymbol a\rangle|: {\boldsymbol a\in\boldsymbol{V},\boldsymbol\theta\in\cup_{b,s}\boldsymbol\Theta_{b,s}}\right\}<\infty$.
	    Moreover, when $\langle\boldsymbol a,\boldsymbol\theta\rangle$ is not constant  
	\begin{flalign}\label{eq:delta_cond}
	\min_{\tilde{s}\in\mathcal{S}}\min_{\substack{(b^*,\boldsymbol{a}^*)\in\mathcal{A}^*(P_S,\mathcal{P}_{\boldsymbol\Theta},\tilde{s}),\\ (b,\boldsymbol{a})\notin \mathcal{A}^*(P_S,\mathcal{P}_{\boldsymbol\Theta},\tilde{s}),\boldsymbol{a}\in\boldsymbol{V}}} [g(b^*,\boldsymbol{a}^*,\tilde{s})-g(b,\boldsymbol{a},\tilde{s})]>0,
	\end{flalign}
	and thus, there exists $\delta>0$ sufficiently small such  that for every $\hat{P}_{S}$ and $\hat{P}_{\boldsymbol\Theta_{b,s}}$ satisfying $|\hat{P}_{S}(\tilde{s},\check{s})-P_S(\tilde{s},\check{s})|\leq\delta$ for all $\tilde{s},\check{s}\in\mathcal{S}$, and
	$|\hat{P}_{\boldsymbol\Theta_{b,s}}(\boldsymbol\theta)-P_{\boldsymbol\Theta_{b,s}}(\boldsymbol\theta)|\leq\delta$ for all  $\boldsymbol\theta\in\boldsymbol\Theta_{b,s}$ and $(b,s)\in\mathcal{B}\times\mathcal{S}$
	we have that
	$\mathcal{A}^*\left(\hat{P}_S,\{\hat{P}_{\boldsymbol\Theta_{b,s}}\}_{ (b,s)\in\mathcal{B}\times\mathcal{S}},\tilde{s}\right)= \mathcal{A}^*\left(P_S,\{P_{\boldsymbol\Theta_{b,s}}\}_{ (b,s)\in\mathcal{B}\times\mathcal{S}},\tilde{s}\right)
	$ for every $\tilde{s}\in\mathcal{S}$.

\end{proof}

Lemma \ref{lemma:equal_vicinity} motivates the development of Algorithm \ref{algo:UCRL_instantaneous} which utilizes upper confidence bounds. 
The lemma ensures, that once the sample probabilities are sufficiently accurate and the state is correctly estimated, we obtain the optimal selection of arm and action.
Additionally, Lemma \ref{lemma:equal_vicinity} 
proves that estimating the probabilities $P_{\boldsymbol\Theta_{b,s}}$ can replace the estimation of the individual reward function for each state arm and action. This has a significant effect on the regret since we assume that $|\boldsymbol V|\gg |\boldsymbol\Theta_{b,s}|$ for all $(b,s)\in\mathcal{B}\times\mathcal{S}$, such is the case for example when $\mathcal{A}$ is an $N$-dimensional cube, in this case the cardinality of $\boldsymbol V$ is exponential in the dimension $N$.  

\subsection{Estimation of the Probability Distributions}
We next discuss the estimation of the probability distributions $P_{S}(\tilde{s},\check{s})$ and  
$P_{\boldsymbol\Theta_{b,s}}(\boldsymbol\theta)$. To this end,
 we define the following notations:
Let $N_{t}(s)$ be the number of occurrences of the state $s$ until time $t-1$. Additionally, let $N_{t}(\tilde{s},\check{s})$ be
 the number of transitions from $\tilde{s}$ to $\check{s}$ until time $t-1$.
Similarly, let $N_{t}(b,s)$ be the number of times the arm $b$ is played and immediately the state $s$ is observed, until time $t-1$. Finally, let  
$N_{t}(b,s,\boldsymbol{\theta})$ be the number of occurrences of $\boldsymbol\theta\in\boldsymbol\Theta_{b,s}$  until time $t-1$.

Algorithm \ref{algo:UCRL_instantaneous} estimates the transition probability  $P_{S}(\tilde{s},\check{s})$ as follows:
\begin{flalign}
\hat{P}_{t,S}(\tilde{s},\check{s})=
\begin{cases}
\frac{N_t(\tilde{s},\check{s})}{N_t(\tilde{s})} & \text{ if } N_t(\tilde{s})>0\\
|\mathcal{S}|^{-1} & \text{ if } N_t(\tilde{s})=0
\end{cases}.
\end{flalign}
Similarly, we  estimate the probability $P_{\boldsymbol\Theta_{b,s}}(\boldsymbol\theta)$ by
\begin{flalign}
\hat{P}_{t,\boldsymbol\Theta_{b,s}}(\boldsymbol\theta)=
\begin{cases}
\frac{N_{t}(b,s,\boldsymbol\theta)}{N_{t}(b,s)} & \text{ if } N_{t}(b,s)>0\\
|\boldsymbol\Theta_{b,s}|^{-1} & \text{ if } N_{t}(b,s)=0
\end{cases}.
\end{flalign}
Denote, 
\begin{flalign}
\text{conf}_{S}(t,s)&\triangleq \min\left\{1,\sqrt{\frac{\log(4(t-1)^{\alpha}|\mathcal{S}|^2)}{2N_{t}(s)}}\right\},\nonumber\\ 
\text{conf}_{\boldsymbol\Theta}(t,b,s)&\triangleq \min\left\{1,\sqrt{\frac{\log(4(t-1)^{\alpha}|\boldsymbol\Theta_{b,s}||\mathcal{B}||\mathcal{S}|)}{2N_{t}(b,s)}}\right\},
\end{flalign}
where $\alpha$ is a constant such that $\alpha>3$.

To evaluate the expected regret of Algorithm \ref{algo:UCRL_instantaneous} we introduce the following lemma.
\begin{lemma}\label{lemma_UCRL_basic_upper}
For every $t>1$, $s,\tilde{s},\check{s}\in\mathcal{S}$, $b\in\mathcal{B}$ and $\boldsymbol\theta\in\boldsymbol\Theta_{b,s}$: 
 	\begin{flalign}\label{eq:chernoff_hoeffdings_basic_conf}
	&\hspace{-0.1cm}\Pr\left(|\hat{P}_{t,S}(\tilde{s},\check{s})-P_S(\tilde{s},\check{s})|>\textup{conf}_{S}(t,\tilde{s})\right)\leq\frac{(t-1)^{-\alpha+1}}{2|\mathcal{S}|^2},\nonumber\\
	&\hspace{-0.1cm}\Pr\left(|\hat{P}_{t,\boldsymbol\Theta_{b,s}}(\boldsymbol\theta)-P_{\boldsymbol\Theta_{b,s}}(\boldsymbol\theta)|>\textup{conf}_{\boldsymbol\Theta}(t,b,s)\right)\leq \frac{(t-1)^{-\alpha}}{2|\boldsymbol\Theta_{b,s}||\mathcal{B}||\mathcal{S}|}.
	\end{flalign}
\end{lemma}
Appendix \ref{proof_lemma_UCRL_basic_upper} proves  (\ref{eq:chernoff_hoeffdings_basic_conf}) which follows  from the union bound, the Markovity of the state process and the Hoeffding inequality.

Using these inequalities we define the confidence intervals
\begin{flalign}\label{eq:condience_bounds_inequality}
|\hat{P}_{t,S}(\tilde{s},\check{s})-P_S(\tilde{s},\check{s})|&\leq \text{conf}_{S}(t,\tilde{s}),\nonumber\\
|\hat{P}_{t,\boldsymbol\Theta_{b,s}}(\boldsymbol\theta)-P_{\boldsymbol\Theta_{b,s}}(\boldsymbol\theta)|&\leq \text{conf}_{\boldsymbol\Theta}(t,b,s)
\end{flalign}
of length $\text{conf}_{S}(t,s)$ and $\text{conf}_{\boldsymbol\Theta}(t,b,s)$, respectively.

We note that the estimations of $P_S$ and $P_{\boldsymbol\Theta_{b,s}}$ require the decision maker to recover the previous state $s_{t-1}$ at each time $t$. In the classical MDP model the state identity  is assumed to be known, however, state recovery schemes are not addressed at all. Interestingly,  our problem structure provides an example for a special form of model in which a state recovery scheme is achievable by exploiting the finite cardinality of the state space,\footnote{In Section \ref{sec:countable_theta} we discuss how to tolerate countable and discrete  sets $\boldsymbol\Theta_{b,s}$. } the finite cardinality  of the sets $\boldsymbol\Theta_{b,s}$ and the fact that $\boldsymbol\Theta_{b,\tilde{s}}\cap\boldsymbol\Theta_{b,\check{s}}=\emptyset\:\forall b\in\mathcal{B},\tilde{s},\check{s}\in\mathcal{S}$  to detects the previous state from the reward with probability 1  while forfeiting a negligible amount of reward.

\subsection{Main Theorem - Upper Bounding the Expected Regret}
Before we upper bound the expected regret for Algorithm \ref{algo:UCRL_instantaneous} we define the following notations. 
Denote $T_M=\max_{\tilde{s},\check{s}\in\mathcal{S}} E(T_{\tilde{s},\check{s}})$ where $T_{\tilde{s},\check{s}}$ is the passage time of first arriving at state $\check{s}$ when starting from state $\tilde{s}$, and let $T_{S}=\left(\min_{\tilde{s},\check{s}\in\mathcal{S}:P_S(\tilde{s},\check{s})>0}\{P_S(\tilde{s},\check{s})\}\right)^{-1}$. Additionally,  denote\footnote{We note that  there is a need for the notation $r_{\max}$ since we do not assume that the maximal reward is $1$.} by $r_{\max} = \max_{\boldsymbol a,\tilde{\boldsymbol a}\in\mathcal{A},\boldsymbol\theta,\tilde{\boldsymbol\theta}\in\bigcup_{(b,s)\in\mathcal{B}\times\mathcal{S}}\boldsymbol\Theta_{b,s}}\left\{\left\langle\boldsymbol a,\boldsymbol \theta\right\rangle-\left\langle\tilde{\boldsymbol a},\tilde{\boldsymbol \theta}\right\rangle\right\}$ the maximal instantaneous regret of any choice of arm-action pair, and let
$C_{\boldsymbol\Theta_{\max}} = \max_{b,s}|\boldsymbol\Theta_{b,s}|$.

\begin{theorem}\label{theorem:man_logarithmic_regret}
	The expected regret of Algorithm \ref{algo:UCRL_instantaneous} is 
	\begin{flalign}\label{eq:log_regret_theorem}
	&O\left(C_{\boldsymbol\Theta_{\max}}|\mathcal{B}||\mathcal{S}|T_MT_Sr_{\max}\frac{\log\left(4T^{\alpha}C_{\boldsymbol\Theta_{\max}}|\mathcal{B}||\mathcal{S}|\right)}{\Delta^2}\right.\nonumber\\
	&\hspace{1.5cm}\left.+C_{\boldsymbol\Theta_{\max}}|\mathcal{B}|^2|\mathcal{S}|^2T_MT_Sr_{\max} \log_2\left(\frac{T}{|\mathcal{S}||\mathcal{B}|}+1\right)\right),
	\end{flalign} 
	where $\alpha$ is a constant such that $\alpha>3$  and $\Delta>0$  is the maximal $\delta$ such that  if
	$|\hat{P}_{S}(\tilde{s},\check{s})-P_S(\tilde{s},\check{s})|\leq\delta$ for all $\tilde{s},\check{s}\in\mathcal{S}$, and
	$|\hat{P}_{\boldsymbol\Theta_{b,s}}(\boldsymbol\theta)-P_{\boldsymbol\Theta_{b,s}}(\boldsymbol\theta)|\leq\delta$ for all $(b,s)\in\mathcal{B}\times\mathcal{S}$
	then
	$\mathcal{A}^*\left(\hat{P}_S,\{\hat{P}_{\boldsymbol\Theta_{b,s}}\}_{ (b,s)\in\mathcal{B}\times\mathcal{S}},\tilde{s}\right)= \mathcal{A}^*\left(P_S,\{P_{\boldsymbol\Theta_{b,s}}\}_{ (b,s)\in\mathcal{B}\times\mathcal{S}},\tilde{s}\right)
	$
	for every $\tilde{s}\in\mathcal{S}$.
\end{theorem}

Theorem \ref{theorem:man_logarithmic_regret} will be proved in the next section. 
We note that the term (\ref{eq:log_regret_theorem}) does not depend on the cardinality of the set $\boldsymbol V$.
 This follows since   the estimation of the transition probabilities is independent of the choice of arm and action whenever we can detect correctly the system state. Additionally, it 
 follows from Lemma \ref{lemma:equal_vicinity} where we prove that we can replace  the estimation of the individual reward function for each state arm and action with the estimation of $P_{\boldsymbol\Theta_{b,s}}$. This has a significant effect on the regret since  $|\boldsymbol V|\gg C_{\boldsymbol\Theta_{\max}} $, therefore, replacing the constant $|\boldsymbol V|$ with $C_{\boldsymbol\Theta_{\max}} $ significantly reduces the expected regret.

Prior to proving Theorem \ref{theorem:man_logarithmic_regret}  we  discuss  Algorithm \ref{algo:UCRL_instantaneous} and consider several adaptions.
First, as we write before,   the estimation of the transition probabilities is independent of the choice of arm and action whenever we can detect correctly the system state. Therefore,
it follows from  Theorem 1.2 in \cite{kontorovich2008}, Theorem 1.1 in \cite{lezaud1998} and the proof of Lemma \ref{lemma_UCRL_basic_upper} that using the estimation for the transition matrix directly instead of using its confidence interval yields the same upper-bound (\ref{eq:log_regret_theorem}) for the regret. Numerical results confirm that optimizing the transition matrix over the confidence interval indeed does not reduce and can even increase the
expected regret since estimating the transition matrix has an exponentially decreasing error but
the confidence bound shrinks slower. Furthermore, we remark that we can remove the term $T_S$ from (\ref{eq:log_regret_theorem}) if we use confidence intervals for the joint probability distribution $P(\tilde{s},b,\check{s},\boldsymbol\theta)\triangleq P_{S}(\tilde{s},\check{s})P_{\boldsymbol\Theta_{b,\check{s}}}(\boldsymbol\theta)$ instead of using separate sets of confidence intervals for estimating the transition matrix of the Markov chain and for the probability distributions of $\boldsymbol\theta$. However, our numerical results show that this may be suboptimal since in this case we can no longer estimate the transition matrix of the Markov chain jointly over all arms but only for all actions over a particular arm.  Finally, we note that we can reduce the constants in eq. \eqref{eq:log_regret_theorem} by adapting the arguments presented in  \cite{Jaksch:2010:NRB:1756006.1859902} for the UCRL2 algorithm instead of those of the UCRL algorithm presented in \cite{Auer07logarithmiconline}.

\begin{algorithm2e*}
	\SetAlgoLined	
	\caption{}\label{algo:UCRL_instantaneous}
	\textbf{Notations:}
	$\epsilon_t=\epsilon\left(10\cdot t^{\alpha_{\epsilon}}\cdot\max_{\boldsymbol\theta\in\bigcup_{(b,s)\in\mathcal{B}\times\mathcal{S}}\boldsymbol\Theta_{b,s}}\left\{\|\boldsymbol\theta\|_1\right\}\right)^{-1},\quad \forall\: t\in\mathbb{N}$, \\ 
	$\text{conf}_{S}(t,s)\triangleq\min\left\{1,\sqrt{\frac{\log(4(t-1)^{\alpha}|\mathcal{S}|^2)}{2N_{t}(s)}}\right\},\quad \forall\: s\in\mathcal{S},t\in\mathbb{N}$,\\
	$\text{conf}_{\boldsymbol\Theta}(t,b,s)\triangleq\min\left\{1,\sqrt{\frac{\log(4(t-1)^{\alpha}|\boldsymbol\Theta_{b,s}||\mathcal{B}||\mathcal{S}|)}{2N_{t}(b,s)}}\right\},\quad \forall\: b\in\mathcal{B},\: s\in\mathcal{S},t\in\mathbb{N}$;
	
	\KwData{  $\mathcal{A},\:\mathcal{S},\:\mathcal{B},\:\boldsymbol\Theta_{b,s}\:\forall (b,s)\in\mathcal{B}\times\mathcal{S}$,\:$\alpha>3$, \:$\epsilon>0$, \:$\alpha_{\epsilon}>1$;  }
	Set $\hat{s}_{-1} = s$ for some $s\in\mathcal{S}$\;
	Set $N_0(\tilde{s},\check{s})=N_0(\tilde{s})=0\quad\forall\tilde{s},\check{s}\in\mathcal{S}$\;
	Set $N_0(b,s)=0\quad\forall\:(b,s)\in\mathcal{B}\times\mathcal{S}$ \;
	Set $\text{conf}_{S}(0,s)=\text{conf}_{S}(1,s)=1\quad \forall s\in\mathcal{S}$\;
	Set $\text{conf}_{\boldsymbol\Theta}(0,b,s)=\text{conf}_{\boldsymbol\Theta}(1,b,s)=1\quad \forall\: b\in\mathcal{B}, s\in\mathcal{S}$ \;
	Set $t=0$\;
	\For{ round $k=0,1\ldots$}{
		\textbf{Initialize round} $k$\textbf{:}
		\begin{enumerate}
			\item Set $t_k=t$;
			\item For every $\tilde{s},\check{s}\in\mathcal{S}$ such that $N_{t_k}(\tilde{s})>0$ set
			$\hat{P}_{t_k}(\tilde{s},\check{s})=\frac{N_{t_k}(\tilde{s},\check{s})}{N_{t_k}(\tilde{s})}$.
			Otherwise, set $\hat{P}_{t_k}(\tilde{s},\check{s})=|\mathcal{S}|^{-1}$;
			
			\item For every $b\in\mathcal{B}$, $s\in\mathcal{S}$ and $\boldsymbol\theta\in \boldsymbol\Theta_{b,s}$ such that $N_{t_k}(b,s)>0$ set $\hat{P}_{t_k,\boldsymbol\Theta_{b,s}}(\boldsymbol\theta) = \frac{N_{t_k}(b,s,\boldsymbol\theta)}{N_{t_k}(b,s)}$. Otherwise, set  
			$\hat{P}_{t_k,\boldsymbol\Theta_{b,s}}(\boldsymbol\theta) = |\boldsymbol\Theta_{b,s}|^{-1}$;
			\item 	Calculate the policy $(b^*_{t_k}(\tilde{s}),\boldsymbol a^*_{t_k}(\tilde{s}))$ for every $\tilde{s}\in\mathcal{S}$, where
			\begin{flalign}\label{policy_update_rule_bounded_arm}
			(b^*_{t_k}(\tilde{s}),\boldsymbol a_{t_k}^*(\tilde{s}))&=
			\arg\max_{\substack{b\in\mathcal{B},\boldsymbol a\in\mathcal{A},\\
			\tilde{P}_{t_k,\boldsymbol\Theta_{b,\check s}}}(\boldsymbol\theta)}\left\{
			\sum_{\check{s}\in\mathcal{S}} \hat{P}_{t_k}(\tilde{s},\check{s})\sum_{\boldsymbol\theta\in\boldsymbol\Theta_{b,\check{s}}}
			\tilde{P}_{t_k,\boldsymbol\Theta_{b,\check s}}(\boldsymbol\theta)\left\langle\boldsymbol a,\boldsymbol \theta\right\rangle\right\}\nonumber\\
			&\text{s.t.}:
			|\tilde{P}_{t_k,\boldsymbol\Theta_{b,\check s}}(\boldsymbol\theta)-\hat{P}_{t_k,\boldsymbol\Theta_{b,\check s}}(\boldsymbol\theta)|\leq \text{conf}_{\boldsymbol\Theta}(t_k,b,\check{s}),\:\forall\: b\in\mathcal{B},\check{s}\in\mathcal{S}\nonumber\\
			&\qquad\tilde{P}_{t_k,\boldsymbol\Theta_{b,\check s}}(\boldsymbol\theta)\geq 0,\: \sum_{\boldsymbol\theta\in\boldsymbol\Theta_{b,\check s}}\tilde{P}_{t_k,\boldsymbol\Theta_{b,\check s}}(\boldsymbol\theta)=1,
			\:\forall\:b\in\mathcal{B},\check{s}\in\mathcal{S},\boldsymbol\theta\in\boldsymbol\Theta_{b,\check s}.
			\end{flalign}
		\end{enumerate}
		
		\textbf{Execute round} $k$\textbf{:}
		
		\While{\begin{itemize}
				\item $\textup{conf}_{S}(t,s)>\textup{conf}_{S}(t_k,s)/2$ for every $s\in\mathcal{S}$, and
				\item $\textup{conf}_{\boldsymbol\Theta}(t,b,s)>\textup{conf}_{\boldsymbol\Theta}(t_k,b,s)/2$ for every $b\in\mathcal{B},s\in\mathcal{S}$
			\end{itemize}
		}{
			\begin{enumerate}
				\item Choose $b_t = b^*_{t_k}(\hat{s}_{t-1})$;

				\item Choose $\boldsymbol a_{t}$ randomly from the set $B(\boldsymbol a_{t_k}^*(\hat{s}_{t-1}),\epsilon_t)\cap\mathcal{A}$;
				
				\item Play the pair $(b_t,\boldsymbol a_{t})$ and observe the reward $r_t$;
				\item Recover system states: set $\hat{\boldsymbol\theta}_{t}\in\bigcup_{s\in\mathcal{S}}\boldsymbol\Theta_{b_t,s}$ to be a solution of $r_t=\left\langle\ \boldsymbol a_t,\hat{\boldsymbol\theta}_{t}\right\rangle$ and set $\hat{s}_{t}\in\mathcal{S}$ to be such that $\hat{\boldsymbol\theta}_{t}\in \boldsymbol\Theta_{b_t,\hat{s}_{t}}$;
				\item Update: 

				\begin{itemize}
					\item Set $N_{t+1}(s)=N_{t}(s)+\mathbbm{1}_{\{s= \hat{s}_{t-1}\}}\mathbbm{1}_{\{t\geq1\}}$;
					\item  Set $N_{t+1}(\tilde{s},\check{s})=N_{t}(\tilde{s},\check{s})+\mathbbm{1}_{\{(\tilde{s},\check{s})= (\hat{s}_{t-1},\hat{s}_{t})\}}\mathbbm{1}_{\{t\geq1\}}$;
					\item Set $N_{t+1}(b,s)=N_{t}(b,s)+\mathbbm{1}_{\{(b,s)= (b_t,\hat{s}_{t})\}}\mathbbm{1}_{\{t\geq1\}}$;
					\item Set $N_{t+1}(b,s,\boldsymbol\theta)=N_{t}(b,s,\boldsymbol\theta)+\mathbbm{1}_{\{(b,s,\boldsymbol\theta)= (b_t,\hat{s}_{t},\hat{\boldsymbol\theta}_t)\}}\mathbbm{1}_{\{t\geq1\}}$;
					\item $t=t+1$;
					\end{itemize}	
		\end{enumerate}}
	}
	  \end{algorithm2e*}

\section{Proof of Theorem \ref{theorem:man_logarithmic_regret}}
The expected regret of Algorithm \ref{algo:UCRL_instantaneous} comprises the following events\footnote{The proof of the union bound appears in Appendix \ref{append:union_bound_error_events}}:
\begin{itemize}
	\item Regret caused by error in state recovery.
	\item Regret caused by suboptimal rounds in which the confidence intervals are larger than $\Delta/2$.
	\item Regret caused by failure of the confidence intervals. 
	\item Regret caused by the deviation of the initial distribution from the stationary distribution of the Markov chain $P_S$. 
\end{itemize}
Next, we show that the expected regret caused by each of these events is no greater than  (\ref{eq:log_regret_theorem}).

\subsection{Regret Caused by Error in State Recovery}\label{sec:state_recovery}
	In the restless hidden Markov bandit model the identity of the previous state is not available to the decision maker, thus the decision maker should balance minimizing the expected regret of the current time and learning the current state of the Markov chain.
Suppose that the decision maker knows $s_{t-1}$, it then chooses at time $t$ the  action $\boldsymbol a_t^{*}=\boldsymbol a_t^{*}(s_{t-1})$ and arm $b_t^{*}=b_t^{*}(s_{t-1})$, calculated in\footnote{See Algorithm \ref{algo:UCRL_instantaneous}.} (\ref{policy_update_rule_bounded_arm}), 
and receives a reward $r_t=\left\langle\boldsymbol a_t^{*},\boldsymbol \theta_t\right\rangle$.
Denote $\boldsymbol\Theta_{b_t^{*}}=\bigcup_{s\in\mathcal{S}}\boldsymbol\Theta_{b_t^{*},s}$.
We distinguish between two cases: 
1) $\boldsymbol\theta_{t}$ is the unique solution of $r_t=\left\langle\boldsymbol a_t^{*},\boldsymbol \theta_t\right\rangle$ in $\boldsymbol\Theta_{b_t^{*}}$.
2) There are multiple solutions to the linear equation $r_t=\left\langle\boldsymbol a_t^{*},\boldsymbol \theta\right\rangle$ in $\boldsymbol\Theta_{b_t^{*}}$.
In the first case, upon receiving the reward $r_t$ the decision maker can fully recover the  vector $\boldsymbol\theta_{t}$ and thus also the system state $s_t$. The decision maker  can then use this information to maximize the expected reward for the next play. In the second case, after receiving the reward the decision maker cannot distinguish between the different vectors that solve the equation $r_t=\left\langle\boldsymbol a_t^{*},\boldsymbol \theta\right\rangle$ in $\boldsymbol\Theta_{b_t^{*}}$.

We overcome this uncertainty by choosing an action $\boldsymbol a_t\in\mathcal{A}$ instead of $\boldsymbol a_t^*$ such that the following conditions hold:
\begin{description}
	\item [\label{cond:A1}{(A1)}] $\boldsymbol a_t\in B(\boldsymbol a_t^*,\epsilon_t)\cap\mathcal{A}$, for some choice of $\epsilon_t>0$.
	\item [\label{cond:A2}{(A2)}] $\left\langle\boldsymbol a_t,\tilde{\boldsymbol \theta}\right\rangle=\left\langle\boldsymbol a_t,\check{\boldsymbol \theta}\right\rangle$ for $\tilde{\boldsymbol \theta},\check{\boldsymbol \theta}\in \boldsymbol\Theta_{b_t^*}$ if and only if $\tilde{\boldsymbol \theta}=\check{\boldsymbol \theta}$.
\end{description}
It is clear that   the first condition can be fulfilled. We prove that 
both conditions (A1) and (A2) can be fulfilled simultaneously.

Let $\mathcal{D}(b_t^*)=\bigcup_{\tilde{\boldsymbol \theta},\check{\boldsymbol \theta}\in\boldsymbol \Theta_{b_t^*}:\tilde{\boldsymbol \theta}\neq\check{\boldsymbol \theta}}\{\boldsymbol a\in \mathcal{A} : \left\langle\boldsymbol a,\tilde{\boldsymbol \theta}\right\rangle=\left\langle\boldsymbol a,\check{\boldsymbol \theta}\right\rangle\}$.
$\mathcal{D}(b_t^*)$ is contained in the union of $|\boldsymbol\Theta_{b_t^*}|(|\boldsymbol\Theta_{b_t^*}|-1)/2$ hyperplanes of dimension $N-1$ whereas the set $B(\boldsymbol a_t^*,\epsilon_t)\cap\mathcal{A}$ is $N$ dimensional.
Therefore the intersection of $\mathcal{D}(b_t^*)$ with $B(\boldsymbol a_t^*,\epsilon_t)\cap\mathcal{A}$ has measure $0$.
Thus, the random choice of the action $\boldsymbol a_t$ from the set $B(\boldsymbol a_t^*,\epsilon_t)\cap\mathcal{A}$ fulfills condition (A2) with probability one. For each such a vector $\boldsymbol a_t$
we have that
$|\left\langle\boldsymbol a_t^*,\boldsymbol \theta\right\rangle-\left\langle\boldsymbol a_t,\boldsymbol \theta\right\rangle|=|\left\langle\boldsymbol a_t^*-\boldsymbol a_t,\boldsymbol \theta\right\rangle|\leq\epsilon_t\cdot\max_{\boldsymbol\theta\in\bigcup_{b\in\mathcal{B}}\boldsymbol\Theta_{b}}\|\boldsymbol\theta\|_1$.
 Finally, 
we choose $\alpha_{\epsilon}>1$ and set $\epsilon_t=\epsilon\left(\gamma\cdot t^{\alpha_{\epsilon}}\cdot\max_{\boldsymbol\theta\in\bigcup_{b\in\mathcal{B}}\boldsymbol\Theta_{b}}\|\boldsymbol\theta\|_1\right)^{-1}$ for all $t$ where $\gamma$ is a constant bigger than the finite sum $\sum_{t=1}^{\infty}t^{\alpha_{\epsilon}}$. Since $\sum_{t=1}^{\infty}\epsilon_t\cdot\max_{\boldsymbol\theta\in\bigcup_{b\in\mathcal{B}}\boldsymbol\Theta_b}\|\boldsymbol\theta\|_1<\epsilon$,  the expected regret caused by the state recovery process is smaller than $\epsilon$ with probability one.

We prove that the event that we cannot recover the previous state uniquely occurs with zero probability. Additionally, we establish in the proof of Lemma \ref{lemma:equal_vicinity} that the maximal instantaneous reward at each time instant is finite. Since we consider discrete time, the set of time instants is countable, and therefore the overall expected regret  caused by  state  estimation error is zero. Furthermore, even though we proved that a state detection error occurs with zero probability, for the sake of completeness of presentation we can add this case to Algorithm \ref{algo:UCRL_instantaneous} and state that if there is a state recovery error ((A2) does not hold) we do not update the terms $N_t(s),N_t(\tilde{s},\check{s}),N_t(b,s)$ and $N_t(b,s,\boldsymbol\theta)$ for the current time.
Following this discussion, hereafter we assume that the previous state $s_{t-1}$ is known to the decision maker when the choice of the arm and action at time $t$ are made.

Finally, we note that in a high dimensional problem where $N\gg1$, assuming that the optimal action is sparse, i.e., only a small number of indices in the optimal action $\boldsymbol{a}_t^*$ at time $t$ is non-zero, and assuming that the vectors in the set $\boldsymbol\Theta_{b,\tilde{s}}$ and the vectors in the set $\boldsymbol\Theta_{b,\check{s}}$ are different in at least one non-zero index of $\boldsymbol{a}_t^*$, we can consider actions in a lower dimensional ball where only the non-zero coordinates of $\boldsymbol{a}_t^*$ are considered instead of the $N$-dimensional ball. This is very helpful, for example, in the cognitive radio communication system we described in the introduction since it requires the receiver to estimate a significantly smaller number of channel states, this reduces the time delays, power consumption and complexity decoding design on the receiver end.
   
\subsection{Regret Caused by Suboptimal Rounds}
Next we bound the expected regret caused by suboptimal rounds in which the lengths of the confidence intervals are greater than $\Delta$. To analyze this expected regret we first present the following propositions.
\begin{proposition}\label{prop:NtoConf}
	Let $t_k$ be the starting time of round $k$.
	For every $s\in\mathcal{S}$  and $t>t_k>0$, if $\textup{conf}_S(t,s)\leq\frac{1}{2}\textup{conf}_S(t_k,s)$, then $N_{t}(s)\geq 4N_{t_k}(s)$. Additionally, 
	for every $s\in\mathcal{S},b\in\mathcal{B}$  and $t,t_k>0$, if $\textup{conf}_{\boldsymbol\Theta}(t,b,s)\leq \frac{1}{2}\textup{conf}_{\boldsymbol\Theta}(t,b,s)$, then $N_{t}(s,b)\geq 4N_{t_k}(s,b)$.
\end{proposition}

\begin{proposition}\label{prop:num_rounds}
In $T$ time instants there are at most 
 $|\mathcal{S}||\mathcal{B}|\left[\log_2\left(\frac{T}{|\mathcal{S}||\mathcal{B}|}+1\right)+1\right]$ rounds.	
\end{proposition}

\begin{proposition}\label{prop:confidence_length}
	If $N_t(s)>\frac{2\log(4(t-1)^{\alpha}|\mathcal{S}|^2)}{\Delta^2}$ then the confidence interval for $s$ is smaller than $\Delta/2$. Further, if $N_t(b,s)>\frac{2\log(4(t-1)^{\alpha}|\boldsymbol\Theta_{b,s}||\mathcal{B}||\mathcal{S}|)}{\Delta^2}$ then the confidence interval for $(b,s)$ is smaller than $\Delta/2$. 
\end{proposition}

Suppose that $k$ is a suboptimal round, then at least one of the following two error events occurs: 
\begin{enumerate}
	\item There exist $\tilde{s},\check{s}\in\mathcal{S}$ such that $|\hat{P}_{t_k}(\tilde{s},\check{s})-P_s(\tilde{s},\check{s})|>\Delta/2$ 
	\item Suppose that the policy for round $k$ chooses the arm $b$ whenever state $\tilde{s}$ is observed, then  there exist $s\in\mathcal{S}$,  an arm $b$ and $\boldsymbol\theta\in\boldsymbol\Theta_{b,s}$ such that $|\hat{P}_{t_k,\boldsymbol\Theta_{b,s}}(\boldsymbol\theta)-P_{\boldsymbol\Theta_{b,s}}(\boldsymbol\theta)|>\Delta/2$.
\end{enumerate}

The expected regret that is caused by the first error event is upper-bounded by the term
\begin{flalign}\label{eq:regret_P_S_matrix}
&4cr_{\max}|\mathcal{S}|T_M\frac{\log(4T^{\alpha}|\mathcal{S}|^2)}{\Delta^2}\nonumber\\
&+2r_{\max}T_M|\mathcal{S}|^2|\mathcal{B}|\log_2\left(\frac{T}{|\mathcal{S}||\mathcal{B}|}+1\right)+r_{\max}|\mathcal{S}|
\end{flalign}
where $c$ is a constant satisfying $c<14$.
This is a direct result of the analysis presented in  \cite{Auer07logarithmiconline} and Propositions \ref{prop:NtoConf}-\ref{prop:confidence_length}. 

The expected regret caused by the second event can be upper bounded as follows. Denote  by $\mathcal{S}_b$ the set of states which upon observing, the decision maker plays the arm $b$. Suppose that there is $\boldsymbol\theta\in\boldsymbol\Theta_{b,s}$ such that $|\hat{P}_{t_k,\boldsymbol\Theta_{b,s}}(\boldsymbol\theta)-P_{\boldsymbol\Theta_{b,s}}(\boldsymbol\theta)|>\Delta/2$ for given $b$ and $s$ such that $\mathcal{S}_b$ is not empty. Let $n(b,s)$ be the number of such rounds  and let $\tau_1(b,s),\ldots,\tau_{n(b,s)}(b,s)$ be their respective lengths. We next upper bound the expected value of the term $\sum_{i=1}^{n(b,s)}\tau_{i}(b,s)$ by dividing each suboptimal round $i$ into $\left\lfloor \frac{\tau_{i}(b,s)}{2T_ST_M}\right\rfloor$ sub-intervals. By the Markov inequality the probability of reaching a state in $\mathcal{S}_b$,  playing the arm $b$, and then immediately reaching the state $s$, is at least $\frac{1}{2}$, for each of these sub-intervals, regardless of the initial state at the beginning of the sub-interval. Let $x_m$ be a binary random variable that is equal to one if  in the $m$th sub-interval arm $b$ was chosen and then the state $s$ was immediately observed, and zero otherwise.
Let $N(b,s,m)=\sum_{i=1}^m X_i$ be the number of such sub-intervals out of $m$.
Then,
\begin{flalign}
&\Pr\left(N(b,s,m)\geq \frac{m}{2}-\sqrt{m\log T}\right)\nonumber\\
&=1-\Pr\left(N(b,s,m) -\frac{m}{2}<-\sqrt{m\log T}\right).
\end{flalign}
Let $Y_m=N(b,s,m)-\frac{m}{2}$, and note that 
$(N(b,s,m)-N(b,s,m-1))\in\{0,1\}$.
It follows that
\begin{flalign}
 |Y_m-Y_{m-1}| = \left|N(b,s,m)-N(b,s,m-1)-\frac{1}{2}\right|\leq \frac{1}{2}.
\end{flalign}
Since $E\left(N(b,s,m)|N(b,s,m-1)\right)\geq \frac{1}{2}$, the sequence $Y_m$ is a submartingale.
Thus, by the Azuma-Hoeffding inequality we have that:
\begin{flalign}
&\Pr\left(N(b,s,m)- \frac{m}{2}<-\sqrt{m\log T}\right)\nonumber\\
&\hspace{2cm}\leq
\exp\left(\frac{-m\log T}{2m/4}\right)\leq\frac{1}{T}.
\end{flalign}
Therefore, 
\begin{flalign}
&\Pr\left(N(b,s,m)\geq \frac{m}{2}-\sqrt{m\log T}\right)\geq 1-\frac{1}{T}.
\end{flalign}

Since the confidence bound is  greater than $\Delta/2$, the next possible policy update with confidence bound smaller than $\Delta/2$ must have a confidence bound in the interval  $[\Delta/4,\Delta/2]$.
By Proposition \ref{prop:confidence_length} we have that 
\[N_T(b,s)<\frac{8\log(4(T-1)^{\alpha}|\boldsymbol\Theta_{b,s}|\mathcal{B}||\mathcal{S}|)}{\Delta^2},\] 
since $N(b,s,m)\leq N_T(b,s)$ it follows that:
$
\sum_{i=1}^{n(b,s)}\left\lfloor \frac{\tau_{i}(b,s)}{2T_ST_M}\right\rfloor\leq c\frac{2\log(4T^{\alpha}|\boldsymbol\Theta_{b,s}||\mathcal{B}||\mathcal{S}|)}{\Delta^2}
$
for some constant $c<14$ with probability $1-\frac{1}{T}$.
It follows that 
\begin{flalign}
&\sum_{i=1}^{n(b,s)}\hspace{-0.1cm}\tau_{i}(b,s)\nonumber\\
&\qquad \leq2T_MT_Sc\frac{2\log(4T^{\alpha}|\boldsymbol\Theta_{b,s}||\mathcal{B}||\mathcal{S}|)}{\Delta^2}+2T_MT_S n(b,s)\nonumber\\
&\qquad\leq 
2T_MT_Sc\frac{2\log(4T^{\alpha}|\boldsymbol\Theta_{b,s}||\mathcal{B}||\mathcal{S}|)}{\Delta^2}\nonumber\\
&\qquad\quad+2T_MT_S |\mathcal{S}||\mathcal{B}|\left[\log_2\left(\frac{T}{|\mathcal{S}||\mathcal{B}|}+1\right)+1\right],
\end{flalign}
with probability $1-\frac{1}{T}$, where the last inequality follows by Proposition \ref{prop:num_rounds}.

Finally, denote  $C_{\boldsymbol\Theta_{\max}} = \max_{b,s}|\boldsymbol\Theta_{b,s}|$, then the expected regret is:
\begin{flalign}\label{overall_regret_2_conf_okay}
&C_{\boldsymbol\Theta_{\max}}|\mathcal{B}||\mathcal{S}|r_{\max}T\frac{1}{T}\nonumber\\
&+4C_{\boldsymbol\Theta_{\max}}|\mathcal{B}||\mathcal{S}|T_MT_Sr_{\max}c\frac{\log(4T^{\alpha}C_{\boldsymbol\Theta_{\max}}|\mathcal{B}||\mathcal{S}|)}{\Delta^2}\nonumber\\
&+2C_{\boldsymbol\Theta_{\max}}|\mathcal{B}|^2|\mathcal{S}|^2T_MT_Sr_{\max} \left[\log_2\left(\frac{T}{|\mathcal{S}||\mathcal{B}|}+1\right)+1\right].
\end{flalign}

\subsection{Regret Caused by Failure of the Confidence Intervals}\label{sec:fail_confidence_bound}
Next we upper bound the expected regret caused by the failure of the confidence intervals, i.e., the probability distributions that we estimate are outside the confidence intervals.

Recall that $t_k$ is the starting time of round $k$;
by (\ref{eq:condience_bounds_inequality}),  the probability that one of the confidence intervals fails in round $k$ is upper bounded by the union bound as follows:
\begin{flalign*}
&|\mathcal{S}|^2\frac{(t_k-1)^{-\alpha+1}}{2|\mathcal{S}|^2}+\sum_{(b,s)\in\mathcal{B}\times\mathcal{S}}|\boldsymbol\Theta_{b,s}|\frac{(t_k-1)^{-\alpha}}{2|\boldsymbol\Theta_{b,s}||\mathcal{B}||\mathcal{S}|}\nonumber\\
&\leq(t_k-1)^{-\alpha+1}.
\end{flalign*}
It follows that the expected regret caused by the failure of the confidence bounds can be upper bounded as follows:
\begin{flalign}
&\sum_{k=1}^{|\mathcal{S}||\mathcal{B}|\left[\log_2\left(\frac{T}{|\mathcal{S}||\mathcal{B}|}+1\right)+1\right]} r_{\max}(t_k-1)^{-\alpha+1}(t_k-t_{k-1}) \nonumber\\
&\qquad\leq r_{\max}\sum_{t=0}^{\infty} t^{-\alpha+2} < \infty
\end{flalign}
where the last inequality follows since $\alpha>3$. Thus, the expected regret caused by the failure of the confidence bounds is bounded.

\subsection{Regret Caused by the Deviation of the Initial Distribution from the Stationary Distribution}
Finally, the expected regret $T\rho(\pi^*)-\sum_{t=1}^TE[r_t]$  depends on the  initial distribution of the Markov chain $P_S$. 
By the analysis of the regret caused by error in state recovery, we recover the identity of the previous state with probability one while causing a bounded regret. Thus, we assume that the decision maker knows the identity of the previous state upon making a decision.
The following lemma\footnote{We prove this lemma in Appendix \ref{append:Lemma_2_proof}.} bounds the regret caused by deviating from the stationary distribution $\mu_S$ of a round of length $T$.
\begin{lemma}\label{lemma:not_stationary_first}
	Assuming that the optimal policy $\pi^*$ is played in the restless Markov bandits model with linear rewards, then
$\sum_{t=1}^{T} E[\rho(\pi^*)-r_t(b_t^*,\boldsymbol a_t^*)] \leq T_M r_{\max}$ where $(b_t^*,\boldsymbol a_t^*) = \pi^*(s_{t-1})$.
\end{lemma}
Thus, by Proposition \ref{prop:num_rounds}, the expected regret caused by deviating initially from the stationary distribution of $P_S$ does not exceed\footnote{
	We note that this regret can be bounded more tightly. Since the Markov chain $P_S$ is aperiodic and irreducible, we can bound the deviation from the stationary distribution using Theorem 4.9 in \cite{opac-b1128575}. However, since the upper bound we derive for this regret event is smaller than  (\ref{overall_regret_2_conf_okay}), we do not reduce it further.  } 
$r_{\max}T_M|\mathcal{S}||\mathcal{B}|\log_2\left[\left(\frac{T}{|\mathcal{S}||\mathcal{B}|}+1\right)+1\right]$.

\section{Numerical Results}
Next, we present numerical results evaluating the performance of Algorithm \ref{algo:UCRL_instantaneous}. We compare the average regret of Algorithm \ref{algo:UCRL_instantaneous} to that of a straightforward implementation of the UCRL algorithm, presented in \cite{Auer07logarithmiconline} using two sets of confidence bounds, one for estimating the probability distribution $P(\tilde{s},b,\boldsymbol a,\check{s})\triangleq P_{S}(\tilde{s},\check{s})$ and one for estimating the expected reward  $r(b,\boldsymbol a,\check{s})=\sum_{\boldsymbol\theta\in\boldsymbol\Theta_{b,\check{s}}}P_{\boldsymbol\Theta_{b,\check{s}}}(\boldsymbol\theta)\left\langle\boldsymbol a,\boldsymbol \theta\right\rangle$ for every $\tilde{s},\check{s}\in\mathcal{S},\boldsymbol a\in\boldsymbol{V},b\in\mathcal{B}$, in addition to recovering the previous state using our state recovering scheme. It is easy to see that the values of the probability distributions $P_{S}(\tilde{s},\check{s})$ and $P_{\boldsymbol\Theta_{b,\check{s}}}(\boldsymbol\theta)$  do not depend on the value of the action $\boldsymbol a$; however, the straightforward use of confidence intervals does not take advantage of this fact and estimates the joint probability distributions for every value of $\boldsymbol a$ using only the measurement of the times when this action is played. Additionally, we compare  Algorithm \ref{algo:UCRL_instantaneous} to its variation in which instead of using two sets of confidence bounds, one for the transition matrix of the states and one for the probability distributions of $\boldsymbol\theta$,  we use a single set of confidence bounds for estimating the joint probability distribution $P(\tilde{s},b,\check{s},\boldsymbol\theta)\triangleq P_{S}(\tilde{s},\check{s})P_{\boldsymbol\Theta_{b,\check{s}}}(\boldsymbol\theta)$  for a given quadruple $(\tilde{s},b,\check{s},\boldsymbol\theta)$. To understand the contribution each of the parts of the probability distribution estimation of Algorithm \ref{algo:UCRL_instantaneous} provides we also compare the regret achieved by Algorithm \ref{algo:UCRL_instantaneous} to that of a partially oblivious UCRL algorithm where the side information is used in the estimation of $P_S$ that is calculated jointly for all arms and actions, however, no side information is used for the structure of the reward and thus the expected reward is estimated for each possible action. Additionally, in this partially oblivious scheme we assume that the decision maker is given the identity of the previous system state, therefore, the reduction in regret,  in this case,  does not include a state recovery scheme.

To evaluate the expected regret of Algorithm \ref{algo:UCRL_instantaneous} and the additional schemes we considered two sets of parameters. To demonstrate the high regret that is caused by ignoring the structural side information when the number of extreme points of $\mathcal{A}$ grows exponentially with the dimension the action vector, the action set $\mathcal{A}$ in both sets of parameter is an $N$ dimensional cube.

\textbf{Setup 1:} $\mathcal{A}=\{0,1\}^2$, $\mathcal{B}=\{1,2\}$, and $\mathcal{S}=\{1,2\}$. 
 $|\boldsymbol\Theta_{b,s}|=2,\: \forall b\in\mathcal{B},s\in\mathcal{S}$.
 
 For this system dimensions, we consider two sets of parameters.
 $ $
 \newline
 \textit{System 1a:}
 The vectors $\boldsymbol \theta\in \boldsymbol\Theta_{b,s}$ were drawn uniformly from the set $\{-7,-6,\ldots,10\}^2$. 
Transition probability: 
$P_S = \begin{pmatrix}
0.4 & 0.6\\
0.75 & 0.25
\end{pmatrix}$. 
Additionally,
\begin{flalign*}
P_{ \boldsymbol\Theta_{b=1,s=1}} =  (0.4, 0.6),
 P_{ \boldsymbol\Theta_{b=2,s=1}} = (0.7, 0.3),\\
 P_{ \boldsymbol\Theta_{b=1,s=2}} = (0.7, 0.3),
 P_{ \boldsymbol\Theta_{b=2,s=2}} = (0.5, 0.5).
 \end{flalign*}
 
 \textit{System 1b:}
 The vectors $\boldsymbol \theta\in \boldsymbol\Theta_{b,s}$ were drawn uniformly from the set $\{-10,-9,\ldots,15\}^2$. 
Transition probability: 
$P_S = \begin{pmatrix}
0.8, 0.2\\
0.45,0.55
\end{pmatrix}$.
Additionally,
\begin{flalign*}
P_{ \boldsymbol\Theta_{b=1,s=1}} =  (0.8, 0.2),
 P_{ \boldsymbol\Theta_{b=2,s=1}} = (0.45, 0.55),\\
 P_{ \boldsymbol\Theta_{b=1,s=2}} = (0.3, 0.7),
 P_{ \boldsymbol\Theta_{b=2,s=2}} = (0.4, 0.6).
 \end{flalign*}

 $ $ \newline
 \textbf{Setup 2:}
$\mathcal{A}=\{0,1\}^5$, $\mathcal{B}=\{1,2,3,4\}$, and $\mathcal{S}=\{1,2,3\}$. $|\boldsymbol\Theta_{b,s}|=2$ for every $b\in\mathcal{B},s\in\mathcal{S}$. 

 For this system dimensions, we consider two sets of parameters.
$ $
 \newline
 \textit{System 2a:}
 The vectors $\boldsymbol \theta\in \boldsymbol\Theta_{b,s}$ were drawn uniformly from the set $\{-7,-6,\ldots,10\}^5$. 
Transition probability: 
 $P_S = \begin{pmatrix}
 0.4 & 0.3&0.3\\
 0.25&0.5 & 0.25\\
 0.3& 0.25& 0.45
 \end{pmatrix}$. 
  Additionally:
 \begin{flalign*}
 &P_{ \boldsymbol\Theta_{b=1,s=1}} =  (0.4, 0.6),\quad
 P_{ \boldsymbol\Theta_{b=2,s=1}} = (0.7, 0.3),\nonumber\\
 &P_{ \boldsymbol\Theta_{b=3,s=1}} =  (0.25, 0.75),\quad
 P_{ \boldsymbol\Theta_{b=4,s=1}} = (0.35, 0.65),\nonumber\\
 &P_{ \boldsymbol\Theta_{b=1,s=2}} = (0.7, 0.3),\quad
 P_{ \boldsymbol\Theta_{b=2,s=2}} = (0.5, 0.5),\nonumber\\
 &P_{ \boldsymbol\Theta_{b=3,s=2}} = (0.2, 0.8),\quad
 P_{ \boldsymbol\Theta_{b=4,s=2}} = (0.45, 0.55),\nonumber\\
 &P_{ \boldsymbol\Theta_{b=1,s=3}} = (0.75, 0.25),\quad
 P_{ \boldsymbol\Theta_{b=2,s=3}} = (0.1, 0.9),\nonumber\\
 &P_{ \boldsymbol\Theta_{b=3,s=3}} = (0.6, 0.4),\quad
 P_{ \boldsymbol\Theta_{b=4,s=3}} = (0.32, 0.68). 
 \end{flalign*}
 $ $
 \newline
 \textit{System 2b:}
 The vectors $\boldsymbol \theta\in \boldsymbol\Theta_{b,s}$ were drawn uniformly from the set $\{-10,-9,\ldots,15\}^5$. 
Transition probability: 
 $P_S = \begin{pmatrix}
 0.25& 0.55& 0.2\\
 0.35& 0.25& 0.4\\
 0.2& 0.1& 0.7
 \end{pmatrix}$. 
  Additionally:
 \begin{flalign*}
 &P_{ \boldsymbol\Theta_{b=1,s=1}} =  (0.8, 0.2),\quad
 P_{ \boldsymbol\Theta_{b=2,s=1}} = (0.45, 0.55),\nonumber\\
 &P_{ \boldsymbol\Theta_{b=3,s=1}} =  (0.9, 0.1),\quad
 P_{ \boldsymbol\Theta_{b=4,s=1}} = (0.6, 0.4),\nonumber\\
 &P_{ \boldsymbol\Theta_{b=1,s=2}} = (0.3, 0.7),\quad
 P_{ \boldsymbol\Theta_{b=2,s=2}} = (0.14, 0.86),\nonumber\\
 &P_{ \boldsymbol\Theta_{b=3,s=2}} = (0.76, 0.24),\quad
 P_{ \boldsymbol\Theta_{b=4,s=2}} = (0.5, 0.5),\nonumber\\
 &P_{ \boldsymbol\Theta_{b=1,s=3}} = (0.4, 0.6),\quad
 P_{ \boldsymbol\Theta_{b=2,s=3}} = (0.72, 0.28),\nonumber\\
 &P_{ \boldsymbol\Theta_{b=3,s=3}} = (0.18, 0.82),\quad
 P_{ \boldsymbol\Theta_{b=4,s=3}} = (0.53, 0.47). 
 \end{flalign*}
 
 We also set the following values  $\epsilon=0.5$, $\alpha=3.1$, $\alpha_{\epsilon}=1.5$, $\gamma=1$. We ran a Monte Carlo simulation with 100 realizations of the sets $\boldsymbol\Theta_{b,s}$, for each such realization we generated 20 realizations of the state sequence, and their respective $\boldsymbol\theta$ given the choice of arm $b$. Finally we set $T=10^6$. 
 
 Figures \ref{fig_numerical_results1} and \ref{fig_numerical_results2}
 depict the average regret of each of the schemes that we mentioned at the beginning of this section, that is,  Algorithm \ref{algo:UCRL_instantaneous}, an adaptation of Algorithm \ref{algo:UCRL_instantaneous} with confidence intervals for $P(\tilde{s},b,\check{s},\boldsymbol\theta)$, a partially oblivious adaptation of the UCRL algorithm where the side information is used only in the estimation of the transiion matrix $P_S$ jointly for all arms and actions given a knowlegde of the previous state, and a straightforward adaptation of the UCRL algorithm \cite{Auer07logarithmiconline}.

\begin{figure}
	\centering
	\vspace{-0.4cm}
    \subfigure[Comparison for Setup 1a.]{
		\includegraphics[scale=0.625]{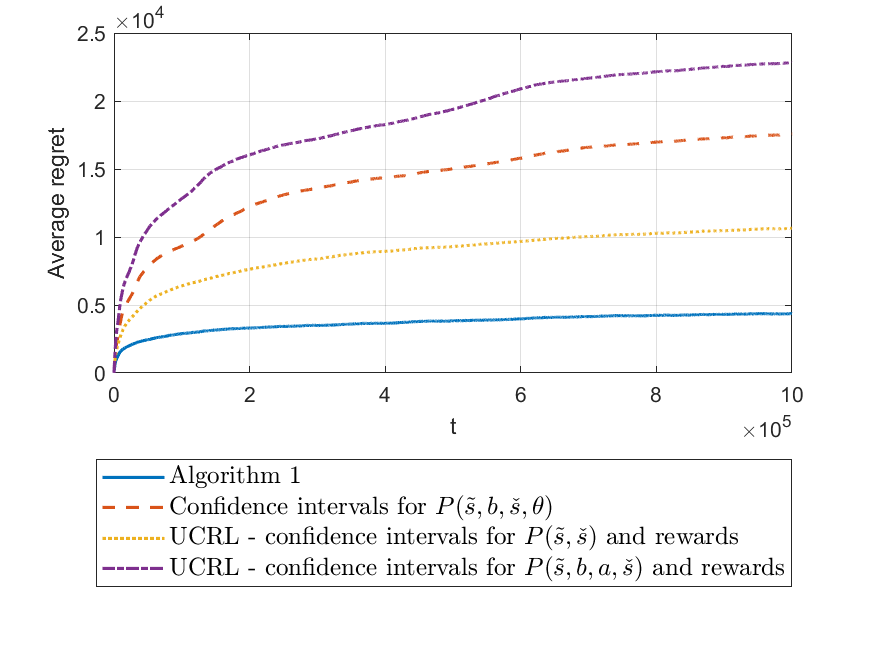}
		\label{fig_numerical_results1a} 
}	
    \subfigure[Comparison for Setup 1b.]{
		\includegraphics[scale=0.625]{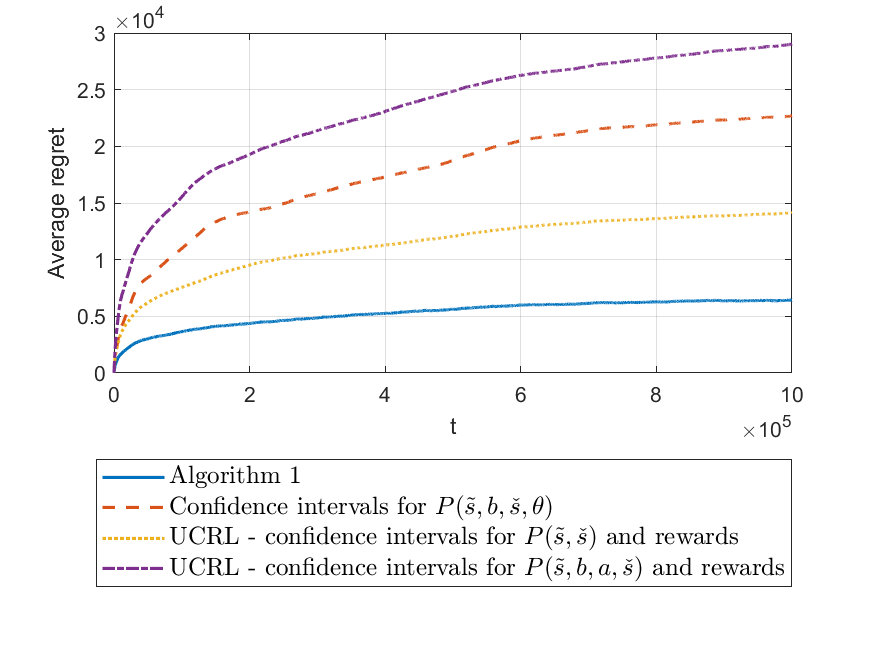}
		\label{fig_numerical_results1b}
	}
	
	\caption[Two numerical solutions]{Comparison between the average regret of three different schemes for Setup 1a and Setup 1b,  Algorithm \ref{algo:UCRL_instantaneous}, an adaptation of Algorithm \ref{algo:UCRL_instantaneous} with confidence intervals for $P(\tilde{s},b,\check{s},\boldsymbol\theta)$,  and a straightforward adaptation of the UCRL algorithm with confidence intervals for $P(\tilde{s},b,\boldsymbol{a},\check{s})$ and the expected reward function $r(b,a,\check{s})$, see \cite{Auer07logarithmiconline}.}
	\label{fig_numerical_results1}
\end{figure}

\begin{figure}
	\centering
	\vspace{-0.4cm}
    \subfigure[Comparison for Setup 2a.]{
		\includegraphics[scale=0.625]{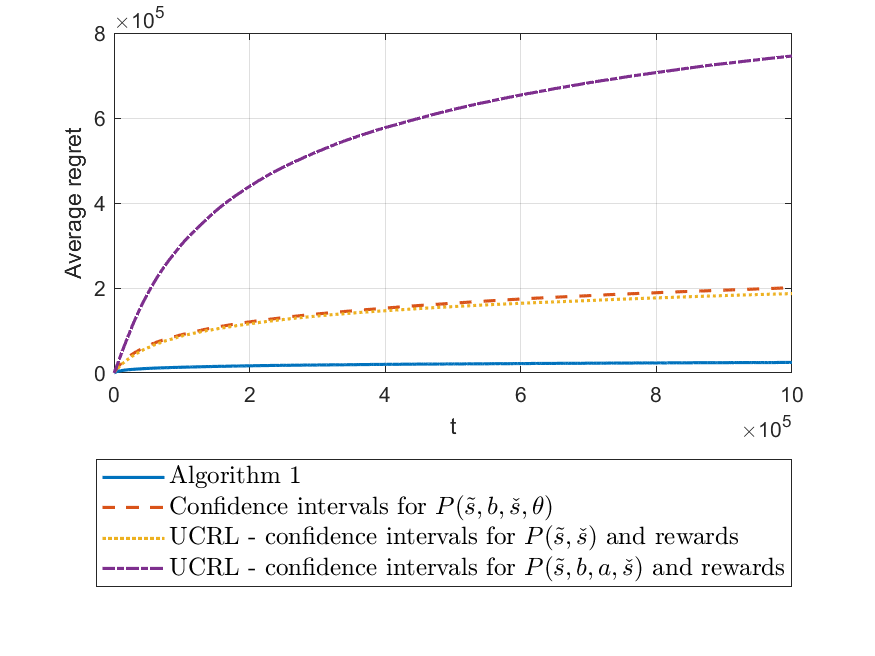}
		\label{fig_numerical_results2a} 
}	
    \subfigure[Comparison for Setup 2b.]{
		\includegraphics[scale=0.625]{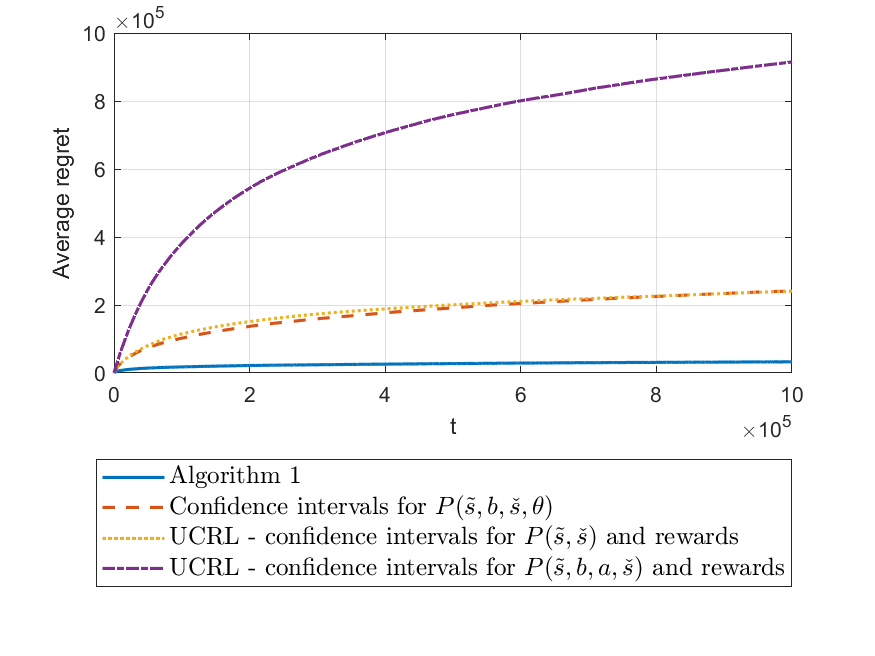}
		\label{fig_numerical_results2b}
	}
	
	\caption[Two numerical solutions]{Comparison between the average regret of four different schemes for Setup 2a and Setup 2b,  Algorithm \ref{algo:UCRL_instantaneous}, an adaptation of Algorithm \ref{algo:UCRL_instantaneous} with confidence intervals for $P(\tilde{s},b,\check{s},\boldsymbol\theta)$, and a straightforward adaptation of the UCRL algorithm with confidence intervals for $P(\tilde{s},b,\boldsymbol{a},\check{s})$ and the expected reward function $r(b,a,\check{s})$, see \cite{Auer07logarithmiconline}}
	\label{fig_numerical_results2}
\end{figure}
 Figures \ref{fig_numerical_results1} and \ref{fig_numerical_results2} show that Algorithm \ref{algo:UCRL_instantaneous} outperforms all the aforementioned possible schemes. This leads to the conclusion that separating the estimation of the probability distributions into two groups, one that is common to all arms (the transition matrix), and one that depends on the identity of the arm played (the probability distribution of $\boldsymbol\theta$) decreases the regret. Additionally, we note that utilizing the information regarding the reward function significantly decreases the regret, in our model it removes the dependency on the cardinality of the action set that may be large. Additionally, we can see that when $|\boldsymbol{V}|=|\boldsymbol\Theta_{b,s}|$ estimating the expected reward for each action and the probabilities $P_{\boldsymbol\Theta_{b,s}}$ yields comparable results as Figure \ref{fig_numerical_results1} demonstrates. However, as we increase the dimension $N$ and the cardinality of $\boldsymbol{V}$ increases exponentially as well, estimating the expected reward for each action yields expected regret that is an order of magnitude higher than the one achieved by estimating probabilities $P_{\boldsymbol\Theta_{b,s}}$.
 Finally, Figures \ref{fig_numerical_results1} and \ref{fig_numerical_results2} confirm that our state recovery scheme is indeed correct. 

\section{Extensions Generalization and Further Discussion}
\subsection{Generalization to Other Reward Functions}
 For the sake of simplicity of presentation this paper analyzes the expected loss function of a linear instantaneous reward function. Next, we show that our analysis holds generally for  convex reward functions that are bounded over the action set $\mathcal{A}$.
Let $g(x)$ be a convex function defined on the polytope $\mathcal{P}$ with an extreme  points set $\boldsymbol{V}$. Then every point in $x\in\mathcal{P}$ is a convex combination of the set of extreme points, i.e., there exist nonnegative weights $(w_v)_{v\in\boldsymbol{V}}$ such that $\sum_{v\in\mathcal{V}}w_v=1$ and $x=\sum_{v\in\boldsymbol{V}}w_v v$. Thus, by the convexity of $g$ 
\begin{flalign}
 g(x) = g\left(\sum_v w_v v\right)\leq \sum_v w_vg(v)\leq \max_{v\in\boldsymbol{V}}g(v),
\end{flalign}
for every $x\in\mathcal{P}$, and the maximum of $g$ in $\mathcal{P}$ is in the vertex set $\boldsymbol{V}$.
Therefore, our analysis holds for any convex and  continuous instantaneous reward function $r_i(\boldsymbol{a},\boldsymbol{\theta})$ that is bounded on the  set of the convex polytope action set $\mathcal{A}$ and for every $b\in\mathcal{B}$ the dimension of the set  $\{\boldsymbol{a}\in\mathcal{A},\boldsymbol{\theta}_1,\boldsymbol{\theta}_2\in\cup_{s\in\mathcal{S}}\boldsymbol\Theta_{b,s}:r_i(b,\boldsymbol{a},\boldsymbol{\theta}_1)=r_i(b,\boldsymbol{a},\boldsymbol{\theta_2})\}$ is at most $N-1$. If the system states $s$ and $\boldsymbol \theta$ are known to the decision maker prior to making a decision the last requirement can be omitted.
Specifically, in the linear case, $r_i(\boldsymbol{a},\boldsymbol{\theta}) = \langle\boldsymbol a, \boldsymbol \theta\rangle$.
Finally, for the sake of clarity of presentation the action set $\mathcal{A}$ does not depend on the choice of arm. However, our analysis can be easily extended to the case where every arm $b$ has its own action set $\mathcal{A}_b$ that is possible to choose from when playing arm $b$, assuming that $\mathcal{A}_b$ is a compact convex polytope for every $b\in\mathcal{B}$.

\subsection{An Alternative Estimation Scheme and Countable and Discrete Sets $\boldsymbol\Theta_{b,s}$}\label{sec:countable_theta}
In this work we assume that $|\boldsymbol{V}|\gg |\boldsymbol\Theta_{b,s}|$ for every $b\in\mathcal{B}$ and $s\in\mathcal{S}$. It follows that, $|\boldsymbol V||\mathcal{S}||\mathcal{B}|\gg \sum_{b\in\mathcal{B},s\in\mathcal{S}}|\boldsymbol\Theta_{b,s}|$ and thus  estimating the probabilities $P_{\boldsymbol\Theta_{b,s}}(\theta)$ instead of the expected regret reduces the  number of estimated variables and leads to smaller regret. However, in the case of large sets $\boldsymbol\Theta_{b,s}$ and in the special case where the sets $\boldsymbol\Theta_{b,s}$ are not finite as in the case of countable and discrete sets estimating the probabilities $P_{\boldsymbol\Theta_{b,s}}(\boldsymbol\theta)$ is not desirable. In this case we can use confidence bounds for the estimation of the $N$-dimensional expected vectors $E(\boldsymbol\theta;b,s)=\sum_{\boldsymbol\theta\in\boldsymbol\Theta_{b,s}}\boldsymbol\theta P_{\boldsymbol\Theta_{b,s}}(\boldsymbol\theta)$. This will significantly reduce the regret whenever $|\boldsymbol{V}|\gg N$ compared with estimating the expected reward for all the actions in the set $\boldsymbol{V}$.

\subsection{Linear Rewards with Additive Noise}
In this paper we consider a noise model that is captured by the random values of  the vectors $\boldsymbol \theta$. We note that similar random noise models are considered in the early works such as  \cite{1104485} by Anantharam et. al. and also by recent works such as \cite{doi:10.1287/moor.2020.1051}. Our paper provides three main contributions: namely, 1) the estimation of the transition matrix $P_{S}$ jointly for all choices of arms and action, 2) the estimation of the probability densities of $\boldsymbol\theta$, or similarly the expected vector values $\sum_{\boldsymbol\theta\in\boldsymbol\Theta_{b,s}} P_{\boldsymbol\Theta_{b, s}}(\boldsymbol\theta)\boldsymbol\theta$  for every arm $b$ and state $s$, instead of the reward function  for every action vector $\boldsymbol a$, arm $b$ and state $s$, 3) the recovery of the previous state $s_{t-1}$. In a scenario where the noise cannot be solely captured by the outcome of the vector $\boldsymbol\theta$, as for example as in the additive reward function 
\begin{flalign}\label{eq:noisy_reward_model}
r_t(b_t,\boldsymbol a_t) = \left\langle\boldsymbol a_t,\boldsymbol \theta(b_t,s_t)\right\rangle+\eta_t,
\end{flalign}
where $\eta_t$ is some random noise process,
we can consider a model where the previous state is known as is assumed, for example, in the classical Markov decision process model analyzed in many works such as \cite{Auer07logarithmiconline,Jaksch:2010:NRB:1756006.1859902,Fruit:2018}. In this case our first two contribution still hold as our numerical results clearly demonstrate.

\subsection{Weakly Communicating Markov Chains}
In this work we consider strongly communicating Markov chain $S$ where all the states are recurrent and propose Algorithm \ref{algo:UCRL_instantaneous} to minimize the expected regret. While the analysis of Algorithm \ref{algo:UCRL_instantaneous} does not apply to a weakly communicating Markov chains where $T_M$ can be infinite, the key contributions that  this paper provides, namely estimating the transition probabilities jointly to all arms and actions, estimating the probabilities $P_{\boldsymbol\Theta_{b,s}}$ instead of the reward for each choice of action and estimating the previous state with bounded regret, can be applied in a similar way to algorithms that consider weakly communicating setups such as \cite{Jaksch:2010:NRB:1756006.1859902,Fruit:2018}.

\section{Conclusion}
This work presented the restless hidden Markov bandit model with linear rewards in which the action of a decision maker does not affect the Markov process that governs the state of the system. Additionally, the system state is not revealed to the decision maker, but rather it is estimated from the previous actions and arms played and their respective rewards.  We showed that by increasing the regret by an arbitrarily small value (independent of $T$) the decision maker can learn the state of the system. Furthermore, we also developed an algorithm that takes advantage of the structural side information, i.e., the linearity of the reward function and the common transition matrix, to yield logarithmic regret that does not depend on the size of the action space (which can be exponential with the number of dimensions).  This is a significant improvement to a naive implementation of an existing algorithm for Markov decision processes and restless Markovian bandits. 

\appendices{} 
\section{Proof of Lemma \ref{lemma_UCRL_basic_upper}}\label{proof_lemma_UCRL_basic_upper}
\begin{proof}[Proof: Part A]
	First, we prove that $\Pr\left(|\hat{P}_{t,S}(\tilde{s},\check{s})-P_S(\tilde{s},\check{s})|>\text{conf}_{S}(t,s)\right)\leq\frac{(t-1)^{-\alpha}}{2|\mathcal{S}|^2}$. 
	By the law of total probability
	\begin{flalign}
		&\Pr\left(|\hat{P}_{t,S}(\tilde{s},\check{s})-P_S(\tilde{s},\check{s})|>\text{conf}_{S}(t,s)\right)\nonumber\\
		&=\sum_{k=0}^{t-1}\Pr\left(|\hat{P}_{t,S}(\tilde{s},\check{s})-P_S(\tilde{s},\check{s})|>\text{conf}_{S}(t,s),N_t(\tilde{s})=k\right)\nonumber\\
		&=
		\Pr\left(\left|\frac{1}{|\mathcal{S}|}-P_S(\tilde{s},\check{s})\right|>1,N_t(\tilde{s})=0\right)\nonumber\\
		&+
		\sum_{k=1}^{t-1}\Pr\left(\left|\frac{N_t(\tilde{s},\check{s})}{k}-P_S(\tilde{s},\check{s})\right|\right.>\nonumber\\
		&\left.\min\left\{1,\sqrt{\frac{\log(4(t-1)^{\alpha}|\mathcal{S}|^2)}{2k}}\right\},N_t(\tilde{s})=k\right).
	\end{flalign}
	Now, if $k=0$, then $\text{conf}_{S}(t,s)=1$ and 
	$\Pr\left(|\hat{P}_{t,S}(\tilde{s},\check{s})-P_S(\tilde{s},\check{s})|>1,N_t(\tilde{s})=k\right)=0$. 
	By definition, $N_t(\tilde{s},\check{s}) = \sum_{i=1}^{t-1}\mathbbm{1}_{\{(s_{i},s_{i+1})=(\tilde{s},\check{s})\}}$  and $N_t(\tilde{s}) = \sum_{i=1}^{t-1}\mathbbm{1}_{\{s_{i}=\tilde{s}\}}$ where $s_i$
	is the state at time $i$. 
	Thus for $k>1$, we have that 
	\begin{flalign}
		&\Pr\left(\left|\frac{N_t(\tilde{s},\check{s})}{k}-P_S(\tilde{s},\check{s})\right|>\right.\nonumber\\
		&\qquad\left.\min\left\{1,\sqrt{\frac{\log(4(t-1)^{\alpha}|\mathcal{S}|^2)}{2k}}\right\},N_t(\tilde{s})=k\right)\nonumber\\
		&=\Pr\left(\left|\frac{\sum_{i=1}^{t-1}\mathbbm{1}_{\{(s_{i},s_{i+1})=(\tilde{s},\check{s})\}}}{k}-P_S(\tilde{s},\check{s})\right|>\right.\nonumber\\
		&\qquad\left.\min\left\{1,\sqrt{\frac{\log(4(t-1)^{\alpha}|\mathcal{S}|^2)}{2k}}\right\},\sum_{i=1}^{t-1}\mathbbm{1}_{\{s_{i}=\tilde{s}\}}=k\right).
	\end{flalign}
	Define by $t_j$ the (random) time of the $j$th occurrence of the state $\tilde{s}$ in the infinite Markovian  sequence $s_1,s_2,\ldots,s_t,s_{t+1}\ldots$. We have that, 
	\begin{flalign}
		&\Pr\left(\left|\frac{1}{k}\sum_{i=1}^{t-1}\mathbbm{1}_{\{(s_{i},s_{i+1})=(\tilde{s},\check{s})\}}-P_S(\tilde{s},\check{s})\right|>\right.\nonumber\\
		&\qquad\left.\min\left\{1,\sqrt{\frac{\log(4(t-1)^{\alpha}|\mathcal{S}|^2)}{2k}}\right\},\sum_{i=1}^{t-1}\mathbbm{1}_{\{s_{i}=\tilde{s}\}}=k\right)\nonumber\\
		&=\Pr\left(\left|\frac{1}{k}\sum_{j=1}^{k}\mathbbm{1}_{\{(s_{t_j},s_{t_j+1})=(\tilde{s},\check{s})\}}-P_S(\tilde{s},\check{s})\right|>\right.\nonumber\\
		&\qquad\left.\min\left\{1,\sqrt{\frac{\log(4(t-1)^{\alpha}|\mathcal{S}|^2)}{2k}}\right\},\sum_{i=1}^{t-1}\mathbbm{1}_{\{s_{i}=\tilde{s}\}}=k\right)\nonumber\\
		&\leq \Pr\left(\left|\frac{1}{k}\sum_{j=1}^{k}\mathbbm{1}_{\{(s_{t_j},s_{t_j+1})=(\tilde{s},\check{s})\}}-P_S(\tilde{s},\check{s})\right|>\right.\nonumber\\
		&\qquad\left.\min\left\{1,\sqrt{\frac{\log(4(t-1)^{\alpha}|\mathcal{S}|^2)}{2k}}\right\}\right)
	\end{flalign}
	
	Now, by the chain rule for the distribution function and the Markovity of the process $(s_i)_{i=1}^{\infty}$ we have that
	\begin{flalign}\label{eq:Markov_time_independent}
		&\Pr\left((t_i,s_{t_i}=\tilde{s},s_{t_i+1})_{i=1}^k\right) \nonumber\\
		& \hspace{1.5cm}= \Pr(t_1)\Pr(s_{t_1+1}|s_{t_1}=\tilde{s})\nonumber\\
		&\hspace{2cm}\qquad\cdot \prod_{i=2}^k\Pr(t_i|s_{t_{i-1}+1})\Pr(s_{t_i+1}|s_{t_i}=\tilde{s})\nonumber\\
		& \hspace{1.5cm}= \Pr(s_{t_1+1}|s_{t_1}=\tilde{s})\prod_{i=2}^k\Pr(s_{t_i+1}|s_{t_i}=\tilde{s})\nonumber\\
		&\hspace{2cm}\qquad\cdot \Pr(t_1) \prod_{i=2}^k\Pr(t_i|s_{t_{i-1}+1}).
	\end{flalign}
	Thus by the law of total probability over $t_1,\ldots,t_k$, it follows that 
	\begin{flalign}
		&\Pr\left(\left|\frac{1}{k}\sum_{j=1}^{k}\mathbbm{1}_{\{(s_{t_j},s_{t_j+1})=(\tilde{s},\check{s})\}}-P_S(\tilde{s},\check{s})\right|>\right.\nonumber\\
		&\qquad\hspace{2cm}\left.\min\left\{1,\sqrt{\frac{\log(4(t-1)^{\alpha}|\mathcal{S}|^2)}{2k}}\right\}\right)\nonumber\\
		&\leq \Pr\left(\left|\frac{1}{k}\sum_{\ell=1}^{k}\mathbbm{1}_{\{(\tilde{s},s_{\ell})=(\tilde{s},\check{s})\}}-P_S(\tilde{s},\check{s})\right|>\right.\nonumber\\
		&\qquad\hspace{2cm}\left.\min\left\{1,\sqrt{\frac{\log(4(t-1)^{\alpha}|\mathcal{S}|^2)}{2k}}\right\}\right)
	\end{flalign}
	where $\mathbbm{1}_{\{(\tilde{s},s_{\ell})=(\tilde{s},\check{s})\}}$ are $k$ i.i.d. Bernoulli random variables with probability of being one  $P_S(\tilde{s},\check{s})$.
	Now, for every $k$ we have that 
	\begin{flalign}
		&\Pr\left(\left|\frac{1}{k}\sum_{i=1}^k\mathbbm{1}_{\{(\tilde{s},s_i)=(\tilde{s},\check{s})\}}-P_S(\tilde{s},\check{s})\right|>\right.\nonumber\\
		&\hspace{2.5cm}\left.\min\left\{1,\sqrt{\frac{\log(4(t-1)^{\alpha}|\mathcal{S}|^2)}{2k}}\right\}\right)\nonumber\\
		&\leq \Pr\left(\left|\frac{1}{k}\sum_{i=1}^k\mathbbm{1}_{\{(\tilde{s},s_i)=(\tilde{s},\check{s})\}}-P_S(\tilde{s},\check{s})\right|\right.\nonumber\\
		&\hspace{3cm}\left. >\sqrt{\frac{\log(4(t-1)^{\alpha}|\mathcal{S}|^2)}{2k}}\right)\nonumber\\
		&\leq \frac{1}{2(t-1)^{\alpha}|\mathcal{S}|^2},
	\end{flalign}
	where the last inequality follows by the Hoeffding inequality.
\end{proof}
\begin{proof}[Proof: Part B]
	Since given a choice of the arm $b$ and the state $\tilde{s}$,  realizations that generated from the distribution $P_{\boldsymbol\Theta_{b,s}}(\boldsymbol\theta)$ are statistically independent. 
	Thus, the inequality 
	\[\Pr\left(|\hat{P}_{t,\boldsymbol\Theta_{b,s}}(\boldsymbol\theta)-P_{\boldsymbol\Theta_{b,s}}(\boldsymbol\theta)|>\text{conf}_{\boldsymbol\Theta}(t,b,s)\right)\leq \frac{(t-1)^{-\alpha}}{2|\boldsymbol\Theta_{b,s}||\mathcal{B}||\mathcal{S}|}\] is derived by straightforward implementation of the Hoeffding inequality and the law of total probability over $N_t(b,s)$. 
\end{proof}
\section{Proofs of Propositions \ref{prop:NtoConf}-\ref{prop:confidence_length}}
\begin{proof}[Proof of Proposition \ref{prop:NtoConf}]
	Recall that $t_k$ is the starting time of round $k$ and that $t>t_k$.
	We separate the proof for the cases of $\text{conf}(t_k,\tilde{\boldsymbol\theta})=1$ and $\text{conf}(t_k,\tilde{\boldsymbol\theta})<1$.
	
	Suppose that $\text{conf}(t_k,s)<1$ and that $\text{conf}(t,s)\leq \frac{1}{2}\text{conf}(t_k,s)$. Then
	\begin{flalign}
		&\sqrt{\frac{\log\left(4(t-1)^{\alpha}|\mathcal{S}|^2\right)}{2N_t(s)}}\leq \frac{1}{2}
		\sqrt{\frac{\log\left(4(t_{k}-1)^{\alpha}|\mathcal{S}|^2\right)}{2N_{t_k}(s)}}\nonumber\\
		&\qquad\Longleftrightarrow\frac{\log\left(4(t-1)^{\alpha}|\mathcal{S}|^2\right)}{N_t(s)}\leq\frac{1}{4}\cdot
		\frac{\log\left(4(t_{k}-1)^{\alpha}|\mathcal{S}|^2\right)}{N_{t_k}(s)}\nonumber\\
		&\Longleftrightarrow 4\cdot\frac{\log\left(4(t-1)^{\alpha}|\mathcal{S}|^2\right)}{\log\left(4(t_{k}-1)^{\alpha}|\mathcal{S}|^2\right)}\leq\frac{N_t(s)}{N_{t_k}
			(s)}
	\end{flalign}
	Since $\frac{\log\left(4(t-1)^{\alpha}|\mathcal{S}|^2\right)}{\log\left(4(t_{k}-1)^{\alpha}|\mathcal{S}|^2\right)}>1$ we have that $\frac{N_t(s)}{N_{t_k}(s)}\geq 4$.
	
	Now, if $\text{conf}(t_k,s)=1$, then
	$\text{conf}(t,s)<\frac{1}{2}$.  Thus, if $N_{t_k}(s)=0$ then  $N_{t}(s)\geq 4N_{t_k}(s)$.
	Else, if $N_{t_k}(s)>0$ then $\sqrt{\frac{\log\left(4(t-1)^{\alpha}|\mathcal{S}|^2\right)}{2N_t(s)}}<\frac{1}{2}\sqrt{\frac{\log\left(4(t_{k}-1)^{\alpha}|\mathcal{S}|^2\right)}{2N_{t_k}(s)}}$, and we concluded above that in this case  $\frac{N_t(s)}{N_{t_k}(s)}\geq 4$.
	
	The proof of the second part of the proposition is similar.
\end{proof}

\begin{proof}[Proof of Proposition \ref{prop:num_rounds}]
	First, note that $N_t(s)\geq N_t(b,s)$ for every $s\in\mathcal{S},b\in\mathcal{B}$. Thus, by Proposition 1, for each round $k$ the shortest possible length of this round is four times the value of $\min_{b,s}\{N_{t_k}(b,s)\}$. It follows that the number of rounds can be upper-bounded by $|\mathcal{S}||\mathcal{B}|\rho_{\max}$ where $\rho_{\max}$ is the smallest positive integer such that 
	$T \leq |\mathcal{S}||\mathcal{B}|\sum_{i=1}^{\rho_{\max}}4^i$.
	It follows that $\rho_{\max}$ is the smallest positive integer greater than $\log_4\left(1+\frac{3T}{|\mathcal{S}||\mathcal{B}|}\right)$. Now, since 
	\begin{flalign}
		\log_4\left(1+\frac{3T}{|\mathcal{S}||\mathcal{B}|}\right)&=\frac{1}{2}\log_2\left(1+\frac{3T}{|\mathcal{S}||\mathcal{B}|}\right)\nonumber\\
		&\leq \log_2\left(1+\frac{T}{|\mathcal{S}||\mathcal{B}|}\right)
	\end{flalign}
	we have that the number of rounds is upper bounded by $|\mathcal{S}||\mathcal{B}|\left[\log_2\left(1+\frac{T}{|\mathcal{S}||\mathcal{B}|}\right)+1\right]$.
\end{proof}

\begin{proof}[Proof of Proposition \ref{prop:confidence_length}]
	This is a direct result of the definitions: $\text{conf}_{S}(t,s)\triangleq \min\left\{1,\sqrt{\frac{\log(4(t-1)^{\alpha}|\mathcal{S}|^2)}{2N_{t}(s)}}\right\}$
	and $\text{conf}_{\boldsymbol\Theta}(t,b,s)\triangleq \min\left\{1,\sqrt{\frac{\log(4(t-1)^{\alpha}|\boldsymbol\Theta_{b,s}||\mathcal{B}||\mathcal{S}|)}{2N_{t}(b,s)}}\right\}$.
	
	If $N_t(s)>\frac{2\log(4(t-1)^{\alpha}|\mathcal{S}|^2)}{\Delta^2}$, then
	\begin{flalign}
		\text{conf}_{S}(t,s)&= \min\left\{1,\sqrt{\frac{\log(4(t-1)^{\alpha}|\mathcal{S}|^2)}{2N_{t}(s)}}\right\}\nonumber\\
		&\leq \min\left\{1,\sqrt{\frac{\log(4(t-1)^{\alpha}|\mathcal{S}|^2)}{2\frac{2\log(4(t-1)^{\alpha}|\mathcal{S}|^2)}{\Delta^2}}}\right\}\nonumber\\
		& = \min\{1,\Delta/2\}\leq \Delta/2.
	\end{flalign}	
	Similarly, if 	
	$N_t(b,s)>\frac{2\log(4(t-1)^{\alpha}|\boldsymbol\Theta_{b,s}||\mathcal{B}||\mathcal{S}|)}{\Delta^2}$, then 
	\begin{flalign}
		\text{conf}_{\boldsymbol\Theta}(t,b,s)&= \min\left\{1,\sqrt{\frac{\log(4(t-1)^{\alpha}|\boldsymbol\Theta_{b,s}||\mathcal{B}||\mathcal{S}|)}{2N_{t}(b,s)}}\right\}\nonumber\\
		&\leq \min\left\{1,\sqrt{\frac{\log(4(t-1)^{\alpha}|\boldsymbol\Theta_{b,s}||\mathcal{B}||\mathcal{S}|)}{2\frac{2\log(4(t-1)^{\alpha}|\boldsymbol\Theta_{b,s}||\mathcal{B}||\mathcal{S}|)}{\Delta^2}}}\right\}\nonumber\\
		& = \min\{1,\Delta/2\}\leq \Delta/2.
	\end{flalign}
	
\end{proof}

\section{Proof of Equation \eqref{eq:regret_P_S_matrix}}
Next we present the lemma that proves \eqref{eq:regret_P_S_matrix}.

\textbf{Lemma:}
The expected regret caused by all suboptimal rounds $k$ such that there exist $\tilde{s},\check{s}\in\mathcal{S}$ such that $|\hat{P}_{t_k}(\tilde{s},\check{s})-P_s(\tilde{s},\check{s})|>\Delta/2$ 
is upper bounded by 
\begin{flalign}
	&4cr_{\max}|\mathcal{S}|T_M\frac{\log(4(T-1)^{\alpha}|\mathcal{S}|^2)}{\Delta^2}\nonumber\\
	&+2r_{\max}T_M|\mathcal{S}|^2|\mathcal{B}|\left[\log_4\left(\frac{T}{|\mathcal{S}||\mathcal{B}|}+1\right)+1\right]+r_{\max}|\mathcal{S}|
\end{flalign}	

\begin{proof}
	Suppose that there exists $\tilde{s}\in\mathcal{S}$ such that $|\hat{P}_{t_k,S}(\tilde{s},\check{s})-P_{S}(\tilde{s},\check{s})|>\Delta/2$ for the state $\check{s}$. Let $n(\tilde{s})$ be the number of such rounds  and let $\tau_1(\tilde{s}),\ldots,\tau_{n(\tilde{s})}$ be their respective lengths. Next we upper bound the expected value of the term $\sum_{i=1}^{n(\tilde{s})}\tau_{i}(\tilde{s})$ by dividing each suboptimal round $i$ into $\left\lfloor \frac{\tau_{i}(\tilde{s})}{2T_S}\right\rfloor$ sub-intervals. By the Markov inequality the probability to visit the state $\tilde{s}$  in a sub-interval is at least  $\frac{1}{2}$, for each of these sub-intervals. Thus, by the Azuma-Hoeffding inequality we have that:
	\begin{flalign}
		\Pr\left(N(\tilde{s},m)\geq \frac{m}{2}-\sqrt{m\log T}\right)\geq 1-\frac{1}{T}
	\end{flalign}
	where $N(\tilde{s},m)$ is the number sub-intervals  in which we visit  state $\tilde{s}$ out of $m$ intervals.
	
	Since the confidence bound is  greater than $\Delta/2$, the next possible policy update with confidence bound smaller than $\Delta/2$ must have a confidence bound in the interval  $[\Delta/4,\Delta/2]$.
	By Proposition 3 we have that 
	$N_T(\tilde{s})<\frac{8\log(4(T-1)^{\alpha}|\mathcal{S}|^2)}{\Delta^2}$, since $N(\tilde{s},m)\leq N_T(\tilde{s})$ it follows that:
	$
	\sum_{i=1}^{n(\tilde{s})}\left\lfloor \frac{\tau_{i}(\tilde{s})}{2T_M}\right\rfloor\leq c\frac{2\log(4T^{\alpha}|\mathcal{S}|^2)}{\Delta^2}
	$
	for some constant $c<14$ with probability $1-\frac{1}{T}$.
	It follows that 
	\begin{flalign}
		&\sum_{i=1}^{n(b,s)}\tau_{i}(b,s)\leq2T_Mc\frac{2\log(4T^{\alpha}|\boldsymbol\Theta_{b,s}||\mathcal{B}||\mathcal{S}|)}{\Delta^2}+2T_M n(\tilde{s})\nonumber\\
		&\leq 
		4T_Mc\frac{\log(4T^{\alpha}|\mathcal{S}|^2)}{\Delta^2}+2T_M |\mathcal{S}||\mathcal{B}|\left[\log_2\left(\frac{T}{|\mathcal{S}||\mathcal{B}|}+1\right)+1\right]
	\end{flalign}

	Finally, by the union bound over $\tilde{s}$ we have that the expected regret caused by suboptimal rounds in which the estimation of the transition probability is inaccurate is upper bounded by: 
	\begin{flalign}
		&4cr_{\max}|\mathcal{S}|T_M\frac{\log(4(T-1)^{\alpha}|\mathcal{S}|^2)}{\Delta^2}\nonumber\\
		&+2r_{\max}T_M|\mathcal{S}|^2|\mathcal{B}|\left[\log_2\left(\frac{T}{|\mathcal{S}||\mathcal{B}|}+1\right)+1\right]+r_{\max}|\mathcal{S}|.
	\end{flalign}
\end{proof}

\section{Proof of Lemma \ref{lemma:not_stationary_first}}\label{append:Lemma_2_proof}
\begin{proof}[Proof of Lemma \ref{lemma:not_stationary_first}]
	Recall the regret definition \eqref{eq:regret_df_compared_to_rho} that  $R(T) = T\rho(\pi^*)-\sum_{t=1}^{T} E[r_t(b_t^*,\boldsymbol a_t^*)]$. We prove Lemma \ref{lemma:not_stationary_first} by bounding the term $\sum_{t=1}^{T} E[r_t(b_t^*,\boldsymbol a_t^*)]$ from below.
	Recall that $(b_{\pi^*}(\tilde{s}),\boldsymbol a_{\pi^*}(\tilde s)) = \pi^*(\tilde s)$. Since $\mu_S$ is the stationary distribution of the Markov chain $P_{S}$ we have that
	\begin{flalign}
		&T\rho(\pi^*) \nonumber\\
		&= \sum_{t=1}^{T}\sum_{\bar s, \tilde s,\check s\in\mathcal{S}}\mu_S(\bar s)P_{S}^{t-1}(\bar s,\tilde s)P_{S}(\tilde s,\check s)\nonumber\\
		&\hspace{3cm}\cdot\sum_{\boldsymbol\theta\in\boldsymbol\Theta_{b_{\pi^*},\check s}}
		P_{\boldsymbol\Theta_{b_{\pi^*}(\tilde{s}),\check s}}(\boldsymbol\theta)\left\langle\boldsymbol a_{\pi^*}(\tilde s),\boldsymbol \theta\right\rangle\nonumber\\
		&= \sum_{\bar s\in\mathcal{S}}\mu_S(\bar s)\sum_{t=1}^{T}\sum_{\tilde{s},\check s\in\mathcal{S}}P_{S}^{t-1}(\bar s,\tilde s)P_{S}(\tilde s,\check s) \nonumber\\
		&\hspace{3cm}\cdot\sum_{\boldsymbol\theta\in\boldsymbol\Theta_{b_{\pi^*},\check s}}
		P_{\boldsymbol\Theta_{b_{\pi^*}(\tilde s),\check s}}(\boldsymbol\theta)\left\langle\boldsymbol a_{\pi^*}(\tilde s),\boldsymbol \theta\right\rangle.
	\end{flalign}
	Thus, there exists $s\in\mathcal{S}$ such that
	\begin{flalign}\label{eq:upper_reward_given_state}
		&\sum_{t=1}^{T}\sum_{\tilde{s},\check{s}\in\mathcal{S}} P_{S}^{t-1}( s,\tilde s)P_{S}(\tilde s,\check s)\nonumber\\
		&\hspace{2cm}\cdot\sum_{\boldsymbol\theta\in\boldsymbol\Theta_{b_{\pi^*},\check s}}
		P_{\boldsymbol\Theta_{b_{\pi^*}(\tilde s),\check s}}(\boldsymbol\theta)\left\langle\boldsymbol a_{\pi^*}( \tilde s),\boldsymbol \theta\right\rangle\geq T\rho(\pi^*).
	\end{flalign}
	Thus, for every $1\leq\tau\leq T$,  
	\begin{flalign}\label{eq:upper_reward_given_state2}
		&\sum_{t=\tau+1}^{T}\sum_{\tilde{s},\check{s}\in\mathcal{S}} P_{S}^{t-1}( s,\tilde s)P_{S}(\tilde s,\check s)\nonumber\\
		&\hspace{2cm}\cdot\sum_{\boldsymbol\theta\in\boldsymbol\Theta_{b_{\pi^*(\tilde{s})},\check s}}
		P_{\boldsymbol\Theta_{b_{\pi^*}(\tilde s),\check s}}(\boldsymbol\theta)\left\langle\boldsymbol a_{\pi^*}(\tilde s),\boldsymbol \theta\right\rangle\nonumber\\
		&\geq T\rho(\pi^*)-\sum_{t=1}^{\tau}\sum_{\tilde{s},\check{s}\in\mathcal{S}} P_{S}^{t-1}( s,\tilde s)P_{S}(\tilde s,\check s)\nonumber\\
		&\hspace{3cm}\cdot\sum_{\boldsymbol\theta\in\boldsymbol\Theta_{b_{\pi^*(\tilde{s})},\check s}}
		P_{\boldsymbol\Theta_{b_{\pi^*}(\tilde s),\check s}}(\boldsymbol\theta)\left\langle\boldsymbol a_{\pi^*}(\tilde s),\boldsymbol \theta\right\rangle.
	\end{flalign}

	Now, let $t_{s}$ be the first occurrence of state $s$ that fulfills (\ref{eq:upper_reward_given_state}), then for every $\bar{s}\in\mathcal{S}$ we have that
	\begin{flalign}\label{eq:upper_reward_given_state3}
		&E\left[\sum_{t=1}^{T}\sum_{\tilde{s},\check{s}\in\mathcal{S}} P_{S}^{t-1}( \bar s,\tilde s)P_{S}(\tilde s,\check s)\right.\nonumber\\
		&\hspace{2cm}\cdot\left.\sum_{\boldsymbol\theta\in\boldsymbol\Theta_{b_{\pi^*(\tilde{s})},\check s}}
		P_{\boldsymbol\Theta_{b_{\pi^*}(\tilde s),\check s}}(\boldsymbol\theta)\left\langle\boldsymbol a_{\pi^*}(\tilde s),\boldsymbol \theta\right\rangle\right]\nonumber\\ 
		&=E_{t_s}\left[E\left(\sum_{t=1}^{T}\sum_{\tilde{s},\check{s}\in\mathcal{S}} P_{S}^{t-1}( \bar s,\tilde s)P_{S}(\tilde s,\check s)\right.\right.\nonumber\\
		&\hspace{2cm}\cdot\left.\left.\sum_{\boldsymbol\theta\in\boldsymbol\Theta_{b_{\pi^*(\tilde{s})},\check s}}
		P_{\boldsymbol\Theta_{b_{\pi^*}(\tilde s),\check s}}(\boldsymbol\theta)\left\langle\boldsymbol a_{\pi^*}(\tilde s),\boldsymbol \theta\right\rangle|t_s\right)\right]\nonumber\\ 
		&=E_{t_s}\left[\sum_{t=1}^{t_s}\sum_{\tilde{s},\check{s}_t\in\mathcal{S}:\check{s}_{t_s}=s} P_{S}^{t-1}(\bar s,\tilde s)P_{S}(\tilde s,\check s_t)\right.\nonumber\\
		&\hspace{2cm}\cdot\left.\sum_{\boldsymbol\theta\in\boldsymbol\Theta_{b_{\pi^*(\tilde{s})},\check s_t}}
		P_{\boldsymbol\Theta_{b_{\pi^*}(\tilde s),\check s_t}}(\boldsymbol\theta)\left\langle\boldsymbol a_{\pi^*}(\tilde s),\boldsymbol \theta\right\rangle\right]\nonumber\\
		&\quad+E_{t_s}\left[\sum_{t=t_s+1}^{T}\sum_{\tilde{s},\check{s}\in\mathcal{S}} P_{S}^{t-1}( s,\tilde s)P_{S}(\tilde s,\check s)\right.\nonumber\\
		&\hspace{2cm}\cdot\left.\sum_{\boldsymbol\theta\in\boldsymbol\Theta_{b_{\pi^*(\tilde{s})},\check s}}
		P_{\boldsymbol\Theta_{b_{\pi^*}(\tilde s),\check s}}(\boldsymbol\theta)\left\langle\boldsymbol a_{\pi^*}(\tilde s),\boldsymbol \theta\right\rangle\right]\nonumber\\
		&\stackrel{(a)}{\geq}  E_{t_s}\left[\sum_{t=1}^{t_s}\sum_{\tilde{s},\check{s}_t\in\mathcal{S}:\check{s}_{t_s}=s} P_{S}^{t-1}(\bar s,\tilde s)P_{S}(\tilde s,\check s_t)\right.\nonumber\\
		&\hspace{2cm}\cdot\left.\sum_{\boldsymbol\theta\in\boldsymbol\Theta_{b_{\pi^*(\tilde{s})},\check s_t}}
		P_{\boldsymbol\Theta_{b_{\pi^*}(\tilde s),\check s_t}}(\boldsymbol\theta)\left\langle\boldsymbol a_{\pi^*}(\tilde s),\boldsymbol \theta\right\rangle\right]\nonumber\\
		&\quad+T\rho(\pi^*)-E_{t_s}\left[\sum_{t=1}^{t_s}\sum_{\tilde s,\check{s}_t\in\mathcal{S}} P_{S}^{t-1}( s,\tilde s)P_{S}(\tilde s,\check s)\right.\nonumber\\
		&\hspace{2cm}\cdot\left.\sum_{\boldsymbol\theta\in\boldsymbol\Theta_{b_{\pi^*(\tilde{s})},\check s}}
		P_{\boldsymbol\Theta_{b_{\pi^*}(\tilde s),\check s}}(\boldsymbol\theta)\left\langle\boldsymbol a_{\pi^*}(\tilde s),\boldsymbol \theta\right\rangle\right]\nonumber\\
		&\stackrel{(b)}{\geq} T\rho(\pi^*)-E(t_s)r_{\max}.
	\end{flalign}
	where the inequality (a) follows from (\ref{eq:upper_reward_given_state2}) 
	and the inequality (b) follows from the notation
	$r_{\max} = \max_{\boldsymbol a,\tilde{\boldsymbol a}\in\mathcal{A},\boldsymbol\theta,\tilde{\boldsymbol\theta}\in\bigcup_{(b,s)\in\mathcal{B}\times\mathcal{S}}\boldsymbol\Theta_{b,s}}\left\{\left\langle\boldsymbol a,\boldsymbol \theta\right\rangle-\left\langle\tilde{\boldsymbol a},\tilde{\boldsymbol \theta}\right\rangle\right\}$
	that appears before Theorem \ref{theorem:man_logarithmic_regret}.

	We can conclude the proof by the following inequalities
	\begin{flalign}
		&T\rho(\pi^*)-\sum_{t=1}^{T} E(r_t(\pi^*(s_{t-1})))\nonumber\\
		&\qquad= T\rho(\pi^*)-\sum_{t=1}^{T} E_{t_s}\{E[r_t(\pi^*(s_{t-1}))|t_s]\} \nonumber\\
		&\quad \stackrel{(a)}{\leq} T\rho(\pi^*)-[T\rho(\pi^*)-E(t_s)r_{\max}]\nonumber\\
		&\quad\stackrel{(b)}{\leq} r_{\max}T_M,
	\end{flalign}
	where (a) follows by (\ref{eq:upper_reward_given_state3}) and since  we assume in Lemma 2 that $(b_t^*,\boldsymbol a_t^*)=\pi^*(s_{t-1})$, and  (b) follows by the notation $T_M=\max_{\tilde{s},\check{s}\in\mathcal{S}} E(T_{\tilde{s},\check{s}})$ that appears before Theorem \ref{theorem:man_logarithmic_regret}. 
\end{proof}

\section{Incorporating the Regret Events to Prove Theorem 1}\label{append:union_bound_error_events}
Next we conclude the proof of Theorem \ref{theorem:man_logarithmic_regret} by explicitly calculating the error probability that is caused by the four regret events we analyzed, that is:
\begin{itemize}
	\item Regret caused by error in state recovery.
	\item Regret caused by suboptimal rounds in which the confidence intervals are larger than $\Delta$/2.
	\item Regret caused by failure of the confidence intervals. 
	\item Regret caused by the deviation of the initial distribution from the stationary distribution of the Markov chain $P_S$. 
\end{itemize}   

Now, by equation \eqref{eq:regret_df_compared_to_rho}, that defines the regret, we have that the expected regret of Algorithm \ref{algo:UCRL_instantaneous} is 
\begin{flalign}\label{eq:exp_regret_def_basic}
	E[R(T)]=E\left[T\rho(\pi^*)-\sum_{t=1}^T r_t(b_t,\boldsymbol a_t)\right],
\end{flalign}
where $b_t,\boldsymbol a_t$ are played according to Algorithm \ref{algo:UCRL_instantaneous} and $\pi^*$ is the optimal policy that maximizes \eqref{algo:UCRL_instantaneous}.
Now, we can rewrite (\ref{eq:exp_regret_def_basic}) as
\begin{flalign}\label{eq:exp_regret_def_basic2}
	&E[R(T)]=E\left[T\rho(\pi^*)-\sum_{t=1}^T r_t(b_t^*,\boldsymbol a_t^*)\right]\nonumber\\
	&\hspace{2.2cm}+E\left[\sum_{t=1}^T r_t(b_t^*,\boldsymbol a_t^*)-\sum_{t=1}^T r_t(b_t,\boldsymbol a_t)\right],
\end{flalign}
where $(b_t^*,\boldsymbol a_t^*)$ denotes playing the optimal policy assuming that the decision maker knows the identity of the previous state and $(b_t,\boldsymbol a_t)$ is the arm and action choices when playing according to Algorithm \ref{algo:UCRL_instantaneous}. 

Now, by the analysis of the regret caused by the deviation of the initial distribution from the stationary distribution of the Markov chain $P_S$, we have that
\begin{flalign}
	&E\left[T\rho(\pi^*)-\sum_{t=1}^T r_t(b_t^*,\boldsymbol a_t^*)\right]\nonumber\\
	&\qquad=T\rho(\pi^*)-E\left[\sum_{t=1}^T r_t(b_t^*,\boldsymbol a_t^*)\right]\nonumber\\
	&\qquad\leq r_{\max}T_M|\mathcal{S}||\mathcal{B}|\left[\log_2\left(\frac{T}{|\mathcal{S}||\mathcal{B}|}+1\right)+1\right].
\end{flalign}
Now, the term $E\left[\sum_{t=1}^T r_t(b_t^*,\boldsymbol a_t^*)-\sum_{t=1}^T r_t(b_t,\boldsymbol a_t)\right]$ depends on the three other regret events, that is, 
\begin{itemize}
	\item Regret caused by error in state recovery.
	\item Regret caused by suboptimal rounds in which the confidence intervals are larger than $\Delta$/2.
	\item Regret caused by failure of the confidence intervals.   
\end{itemize}
We prove in Section \ref{sec:state_recovery} that the expected regret  caused by first event is bounded, i.e., $O(1)$, the expected regret of the second event is upper bounded by
\begin{flalign}
	&4T_M|\mathcal{S}|c\frac{\log(4T^{\alpha}|\mathcal{S}|^2)}{\Delta^2}\nonumber\\
	&+2T_M |\mathcal{S}|^2|\mathcal{B}|\left[\log_2\left(\frac{T}{|\mathcal{S}||\mathcal{B}|}+1\right)+1\right]\nonumber\\
	&+r_{\max}|\mathcal{S}|+4C_{\boldsymbol\Theta_{\max}}|\mathcal{B}||\mathcal{S}|T_MT_Sr_{\max}c\frac{\log(4T^{\alpha}C_{\boldsymbol\Theta_{\max}}|\mathcal{B}||\mathcal{S}|)}{\Delta^2}\nonumber\\
	&+2C_{\boldsymbol\Theta_{\max}}|\mathcal{B}|^2|\mathcal{S}|^2T_MT_Sr_{\max} \left[\log_2\left(\frac{T}{|\mathcal{S}||\mathcal{B}|}+1\right)+1\right]\nonumber\\
	&+C_{\boldsymbol\Theta_{\max}}|\mathcal{B}||\mathcal{S}|r_{\max}T\frac{1}{T}.
\end{flalign} 
Additionally, we prove in Section \ref{sec:fail_confidence_bound} that the expected regret caused by third event is bounded, i.e., $O(1)$.

This proves that the expected regret of Algorithm \ref{algo:UCRL_instantaneous} is:

\begin{flalign}
	&O\left(C_{\boldsymbol\Theta_{\max}}|\mathcal{B}||\mathcal{S}|T_MT_Sr_{\max}\frac{\log\left(4T^{\alpha}C_{\boldsymbol\Theta_{\max}}|\mathcal{B}||\mathcal{S}|\right)}{\Delta^2}\right.\nonumber\\
	&\hspace{2cm}\left.+C_{\boldsymbol\Theta_{\max}}|\mathcal{B}|^2|\mathcal{S}|^2T_MT_Sr_{\max} \log_2\left(\frac{T}{|\mathcal{S}||\mathcal{B}|}+1\right)\right).
\end{flalign}

\bibliographystyle{IEEEtran}

\end{document}